\documentclass{article}
\usepackage{kr}

\usepackage{times}
\usepackage{soul}
\usepackage{url}
\usepackage[hidelinks]{hyperref}
\usepackage[utf8]{inputenc}
\usepackage[small]{caption}
\usepackage{graphicx}
\usepackage{amsmath}
\usepackage{amsthm}
\usepackage{booktabs}

\usepackage{enumerate}
\usepackage[linesnumbered,boxed]{algorithm2e}
\usepackage{amstext}
\usepackage{amssymb}
\usepackage{color}

\usepackage{tikz}
\usetikzlibrary{automata}

\newtheorem{example}{Example}
\newtheorem{theorem}{Theorem}

\title{On Sufficient and Necessary Conditions in Bounded \CTL: A Forgetting Approach}

\author{%
Renyan Feng$^{1,2}$\and
Erman Acar$^{2,*}$\and
Stefan Schlobach$^{2}$\and
Yisong Wang$^{1,}$\footnote{Corresponding author(s).} \and
Wanwei Liu $^3$\\
\affiliations
$^{1}$Guizhou University, P. R. China\\
$^{2}$Vrije Universiteit Amsterdam, Netherlands\\
$^{3}$National University of Defense Technology, P. R. China\\
\emails
fengrenyan@gmail.com,
\{Erman.Acar, k.s.schlobach\}@vu.nl,
yswang@gzu.edu.cn,
wwliu@nudt.edu.cn
}

\begin{document}

\newcommand{\tuple}[1]{{\langle{#1}\rangle}}
\newcommand{\Mod}{\textit{Mod}}
\newcommand\ie{{\it i.e. }}
\newcommand\eg{{\it e.g.}}
\renewcommand{\st}{s.t.}
\newtheorem{definition}{Definition}
\newtheorem{lemma}{Lemma}
\newtheorem{proposition}{Proposition}
\newtheorem{corollary}[theorem]{Corollary}
\newcommand{\rto}{\rightarrow}
\newcommand{\lto}{\leftarrow}
\newcommand{\lrto}{\leftrightarrow}
\newcommand{\Rto}{\Rightarrow}
\newcommand{\Lto}{\Leftarrow}
\newcommand{\LRto}{\Leftrightarrow}
\newcommand{\Var}{\textit{Var}}
\newcommand{\Forget}{\textit{Forget}}
\newcommand{\KForget}{\textit{KForget}}
\newcommand{\TForget}{\textit{TForget}}
\newcommand{\Fst}{\textit{Fst}}
\newcommand{\dep}{\textit{dep}}
\newcommand{\term}{\textit{term}}
\newcommand{\literal}{\textit{literal}}

\newcommand{\Atom}{\mathcal{A}}
\newcommand{\SFive}{\textbf{S5}}
\newcommand{\MPK}{\textsc{k}}
\newcommand{\MPB}{\textsc{b}}
\newcommand{\MPT}{\textsc{t}}
\newcommand{\MPA}{\forall}
\newcommand{\MPE}{\exists}

\newcommand{\DNF}{\textit{DNF}}
\newcommand{\CNF}{\textit{CNF}}

\newcommand{\degree}{\textit{degree}}
\newcommand{\sunfold}{\textit{sunfold}}

\newcommand{\Pos}{\textit{Pos}}
\newcommand{\Neg}{\textit{Neg}}
\newcommand\wrt{{\it w.r.t.}}
\newcommand{\Hm} {{\cal M}}
\newcommand{\Hw} {{\cal W}}
\newcommand{\Hr} {{\cal R}}
\newcommand{\Hb} {{\cal B}}
\newcommand{\Ha} {{\cal A}}

\newcommand{\Dsj}{\triangledown}

\newcommand{\wnext}{\widetilde{\bigcirc}}
\newcommand{\nex}{\bigcirc}
\newcommand{\ness}{\square}
\newcommand{\qness}{\boxminus}
\newcommand{\wqnext}{\widetilde{\circleddash}}
\newcommand{\qnext}{\circleddash}
\newcommand{\may}{\lozenge}
\newcommand{\qmay}{\blacklozenge}
\newcommand{\unt} {{\cal U}}
\newcommand{\since} {{\cal S}}
\newcommand{\SNF} {\textit{SNF$_C$}}
\newcommand{\start}{\textbf{start}}
\newcommand{\Elm}{\textit{Elm}}
\newcommand{\simp}{\textbf{simp}}
\newcommand{\nnf}{\textbf{nnf}}

\newcommand{\CTL}{\textrm{CTL}}
\newcommand{\Ind}{\textrm{Ind}}
\newcommand{\Tran}{\textrm{Tran}}
\newcommand{\Sub}{\textrm{Sub}}
\newcommand{\NI}{\textrm{NI}}
\newcommand{\Inst}{\textrm{Inst}}
\newcommand{\Com}{\textrm{Com}}
\newcommand{\Rp}{\textrm{Rp}}
\newcommand{\forget}{{\textsc{f}_\CTL}}
\newcommand{\ALL}{\textsc{a}}
\newcommand{\EXIST}{\textsc{e}}
\newcommand{\NEXT}{\textsc{x}}
\newcommand{\FUTURE}{\textsc{f}}
\newcommand{\UNTIL}{\textsc{u}}
\newcommand{\GLOBAL}{\textsc{g}}
\newcommand{\UNLESS}{\textsc{w}}
\newcommand{\Def}{\textrm{def}}
\newcommand{\IR}{\textrm{IR}}
\newcommand{\Tr}{\textrm{Tr}}
\newcommand{\dis}{\textrm{dis}}
\def\PP{\ensuremath{\textbf{PP}}}
\def\NgP{\ensuremath{\textbf{NP}}}
\def\W{\ensuremath{\textbf{W}}}
\newcommand{\Pre}{\textrm{Pre}}
\newcommand{\Post}{\textrm{Post}}

\newcommand{\CTLsnf}{{\textsc{SNF}_{\textsc{ctl}}^g}}
\newcommand{\ResC}{{\textsc{R}_{\textsc{ctl}}^{\succ, S}}}
\newcommand{\CTLforget}{{\textsc{F}_{\textsc{ctl}}}}
\newcommand{\degex}{{\textsc{def}_{\textsc{ex}}}}
\newcommand{\Refine}{\textsc{Refine}}
\newcommand{\cf}{\textrm{cf.}}
\newcommand{\NEXP}{\textmd{\rm NEXP}}
\newcommand{\EXP}{\textmd{\rm EXP}}
\newcommand{\coNEXP}{\textmd{\rm co-NEXP}}
\newcommand{\NP}{\textmd{\rm NP}}
\newcommand{\coNP}{\textmd{\rm co-NP}}
\newcommand{\Pol}{\textmd{\rm P}}
\newcommand{\BH}[1]{\textmd{\rm BH}_{#1}}
\newcommand{\coBH}[1]{\textmd{\rm co-BH}_{#1}}
\newcommand{\Empty}{\emptyset}
\newcommand{\NLOG}{\textmd{\rm NLOG}}
\newcommand{\DeltaP}[1]{\Delta_{#1}^{p}}
\newcommand{\PIP}[1]{\Pi_{#1}^{p}}
\newcommand{\SigmaP}[1]{\Sigma_{#1}^{p}}

\maketitle

\begin{abstract}
Computation Tree Logic (\CTL) is one of the central formalisms in formal verification. As a specification language, it is used to express a property
that the system at hand is expected to satisfy. From
both the verification and the system design points
of view, some information content of such property might become irrelevant for the system due to
various reasons, e.g., it might become obsolete by
time, or perhaps infeasible due to practical difficulties. Then, the problem arises on how to subtract such piece of information without altering the
relevant system behaviour or violating the existing
specifications over a given signature. Moreover, in such a scenario, two
crucial notions are informative: the strongest necessary condition (SNC) and the weakest sufficient
condition (WSC) of a given property.

To address such a scenario in a principled way, we
introduce a forgetting-based approach in \CTL\ and
show that it can be used to compute SNC and WSC
of a property under a given model and over a given signature. We study its
theoretical properties and also show that our notion
of forgetting satisfies existing essential postulates of knowledge forgetting.
Furthermore, we analyse the computational complexity of some basic reasoning tasks for the fragment $\CTL_{\ALL\FUTURE}$ in particular.
\end{abstract}

\section{Introduction}
\label{introduction}

Computation Tree Logic (\CTL)~\cite{clarke1981design} is one of the central formalisms in formal verification. As a specification language, it is used to express a property
that the system at hand is expected to satisfy. From
both the verification and the system design points
of view,  there might be situations in which some information content of such property might become irrelevant for the system due to various reasons e.g., it might be discarded or become obsolete by time, or just  become infeasible due to practical difficulties. As keeping such information would be highly space-inefficient, the problem arises on how to remove it without altering the
relevant system behaviour or violating the existing system
specifications over a given signature. Consider the following example.

\begin{example}[Car-Manufacturing Company]\label{car_manufacturing}
Assume a car-manufacturing company which produces two types of cars: a (se)dan car and a (sp)orts car. In each manufacturing cycle, the company has to (s)elect one of the three options: (1) produce $se$  first, and then $sp$; (2) produce $sp$ first, and then $se$; (3) produce $se$ and $sp$ at the same time. At the end of each selection, a final (d)ecision is taken.

In Figure~\ref{BVM}, this scenario is  represented by the Kripke structure $\Hm=(S, R, L)$ with the initial state $s_0$ (called labelled state transition graph),  and the corresponding atomic variables $V=\{d,s,se,sp\}$.
\begin{figure}[ht]
  \centering
  \includegraphics[width=3cm]{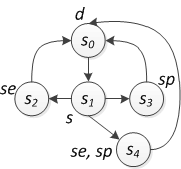}\\
  \caption{Car Engine Manufacturing Scenario }\label{BVM}
\end{figure}
Now assume a situation in which due to some problems (e.g., economic crises or new environmental regulations on the engine technology) company can no longer support the production of sports cars.
This means, all the manufacturing processes concerning $sp$ are no more necessary and should be dropped from both the specifications and the Kripke structure for simplification.
\end{example}

Similar scenarios  like the one presented in Example~\ref{car_manufacturing} may arise in many different domains such as business-process modelling, software development, concurrent systems and more~\cite{Baier:PMC:2008}. Yet dropping some restrictions in a large and complex  system or specification, without affecting the working system components or violating dependent specifications over a given signature, is a non-trivial task.   Moreover, in such a scenario, two logical notions introduced by E. Dijkstra in~\cite{DBLP:journals/cacm/Dijkstra75} are highly informative: the \emph{strongest necessary condition} (SNC) and the \emph{weakest sufficient condition}  (WSC)  of a given specification. These correspond to the \emph{most general consequence} and the \emph{most specific abduction} of such specification, respectively.

 To address these scenarios and to target the relevant notions SNC and WSC in a principled way, we employ a  method based on formal verification.\footnote{ This is  especially useful for abstracting away the domain-dependent problems, and focusing on conceptual ones.} In particular,
we introduce a \emph{forgetting}-based approach in \CTL\, and show that it can
be used to compute SNC and WSC on a restricted subset of the propositional variables, in the same spirit of~\cite{DBLP:Lin:AIJ:2001,doherty2001computing}.

  The rest of the paper is organised as follows. Next section reports about the related work. Section~\ref{preliminaries} introduces the notation and technical preliminaries. As key contributions, Section~\ref{forgetting}, introduces the notion of forgetting in bounded \CTL. Moreover, it provides a model-theoretic characterization of \CTL\  for (initial) Kripke structures, and studies the semantic properties of forgetting. In addition, a complexity analysis, concerning a relevant fragment  $\CTL_{\ALL\FUTURE}$, is carried out.
Section~\ref{ns_conditions} explores the relation between forgetting and SNC (WSC). Section~\ref{section_algorithm} gives a model-based algorithm for computing forgetting in \CTL\ and outline its complexity. Conclusion closes the paper.

Due to space restrictions, for most of the technical results, the actual proof is moved to the supplementary material~\footnote{https://github.com/fengrenyan/proof-of-CTL.git}, and instead an intuitive justification is put in place.

\section{Related Work}\label{related_work}
The notions of SNC and WSC were considered in the scope of formal verification among others,  in generating counterexamples~\cite{dailler2018instrumenting} and refinement of  system~\cite{woodcock1990refinement}.
In addition, the WSC and SNC provide a method to generate successor state axioms from causal theories. 
In~\cite{DBLP:Lin:AIJ:2001}, the SNC and WSC for a proposition $q$ on a restricted subset of the propositional variables under
a propositional theory $T$ are computed based on the notion of forgetting.
Besides, the SNC and WSC are generalized to first order logic (FOL) and a direct method that is based
on \emph{Second-Order Quantifier Elimination} (SOQE) technique has been proposed to automatically generate SNC and WSC in~\cite{doherty2001computing}.

\emph{Forgetting}, which was first formally defined in propositional and FOL by Lin and Reiter~\cite{lin1994forget,eiter2019brief}, can be traced back to the work of Boole on propositional
variable elimination and the seminal work of Ackermann~\cite{ackermann1935untersuchungen}.
Usually, the definition of forgetting can be defined from the perspective of Strong/Semantic Forgetting and Weak Forgetting  respectively~\cite{Zhang:KR:2010}.


In FOL, forgetting has often been studied as an instance of the SOQE problem. It is shown  in~\cite{lin1994forget} that the result of (strongly) forgetting an $n$-ary predicate $P$ from a FOL formula $\varphi$ is $\exists R \varphi[P/R]$, in which $R$ is an $n$-ary predicate variable and $\varphi[X/Y]$ is a result of replacing every occurrence of $X$ in $\varphi$ by $Y$.
The task of forgetting in FOL is to find a first-order formula that is equivalent to $\exists R \varphi[P/R]$.
It is obvious that this is a SOQE problem.
Similarly, the forgetting in description logics (DL) are also explored to create restricted views
of ontologies by eliminating concept and role symbols from DL-based
ontologies~\cite{Wang:AMAI:2010,Lutz:IJCAI:2011,Zhao:2017:IJCAI}.

In propositional logic (PL), forgetting has often been studied under the name of variable
elimination. In particular, the solution of forgetting a propositional variable $p$ from a PL formula $\varphi$ is $\varphi[p/\bot] \vee \varphi[p/\top]$~\cite{lin1994forget}. 
In~\cite{Yan:AIJ:2009}, the authors define the knowledge forgetting of \SFive\ modal logic from the \emph{strong forgetting} point of view to explore the relation between knowledge forgetting and knowledge update. Besides, they have proposed four general postulates (as we will revisit) for knowledge forgetting and shown that these four postulates precisely characterize the notion of knowledge forgetting described above in \SFive.
Moreover, forgetting in logic programs under answer-set semantics are considered in~\cite{DBLP:Zhang:AIJ2006,DBLP:journals/ai/EiterW08,Wong:PhD:Thesis,Yisong:JAIR,Yisong:IJCAI:2013}.

However, existing forgetting definitions in PL and answer set programming are not directly applicable in modal logics.
Moreover, existing forgetting techniques are not directly applicable in \CTL\ either because there are some temporal operators in \CTL\ but not in \SFive.
Similar to~\cite{Yan:AIJ:2009}, we research forgetting in bounded \CTL\ from the semantic forgetting point of view and show that the result of forgetting some propositions from a \CTL\ formula is always expressible in \CTL.
Furthermore, we show that our notion of forgetting satisfies those four postulates of forgetting presented in~\cite{Yan:AIJ:2009}.
And last, we demonstrate how forgetting can be used to compute the SNC and WSC on a set of the propositions.


\section{Notation and Preliminaries}
\label{preliminaries}
 Throughout this paper, we fix a finite set $\Ha$ of propositional variables (or atoms or propositions), use $V$, $V'$ for subsets of $\Ha$ and $\overline V = \Ha - V$.

\subsection{Kripke structures in $\CTL$}
In general, a transition system
 can be described by a \emph{Kripke \ structure} (see~\cite{Baier:PMC:2008} for details). A Kripke structure is a triple $\Hm=(S,R,L)$~\cite{emerson1990temporal}, where
\begin{itemize}
  \item $S$ is a finite nonempty set of states,\footnote{Since \CTL\ has finite model property~\cite{DBLP:journals/jcss/EmersonH85} we assume that the signature of states
  is fixed and finite, i.e., $S\subseteq\cal S$ with ${\cal S}=\{b_1,\ldots,b_m\}$,
  such that any \CTL\ formula with bounded length is satisfiable if and only if it is satisfiable in a such Kripke structure. Thus, there are only finite number of Kripke structures. },
  	
  \item $R\subseteq S\times S$ and, for each $s\in S$, there
  is $s'\in S$ such that $(s,s')\in R$,
  \item $L: S\rto 2^{\cal A}$ is a labeling function.
\end{itemize}

Given a Kripke structure $\Hm=(S,R,L)$, a \emph{path} $\pi$ of $\Hm$ is an infinite sequence
$\pi=(s_0, s_{1} s_{2},\dots)$ of states with
$(s_j, s_{j+1}) \in R$ for every $j\ge 0$.
By $s'\in \pi$, we mean that $s'$ is a state occurring in the path $\pi$.
In particular, we call $\pi_{s}$ 
a path of $\Hm$ starting  from $s$.
A state $s$ is {\em initial} if  there is a path $\pi_s$ of ${\cal M}$ \st\ $s'\in \pi_s$ for each state $s'\in S$.
If $s_0$ is an initial state of $\Hm$, then we denote this Kripke structure $\Hm$ as $(S,R,L,s_0)$ and call it an \emph{initial structure}.

For a given initial structure $\Hm=(S,R,L,s_0)$ and $s\in S$,
the {\em computation tree}
$\Tr_n^{\cal M}(s)$ of $\cal M$ (or simply $\Tr_n(s)$), that has depth $n$ and is rooted at $s$, is recursively defined as in~\cite{browne1988characterizing}, for $n\ge 0$,
\begin{itemize}
  \item $\Tr_0(s)$ consists of a single node $s$ with label $L(s)$.
  \item $\Tr_{n+1}(s)$ has as its root a node $s$ with label  $L(s)$, and
  if $(s,s')\in R$ then the node $s$ has a subtree $\Tr_n(s')$.
\end{itemize}

A {\em \MPK-structure} (or {\em \MPK-interpretation}) $\mathcal{K}$ consists of an initial structure
${\cal M}=(S, R, L, s_0)$ and a state $s\in S$, i.e.,  $\mathcal{K} = (\mathcal{M}, s)$.
If in addition $s=s_0$ (i.e., $\mathcal{K} = (\mathcal{M}, s_0)$), then the \MPK-structure is called an {\em initial} \MPK-structure.

\subsection{Syntax and Semantics of \CTL}
In the following we briefly review the basic syntax and semantics
of the \CTL~\cite{DBLP:journals/toplas/ClarkeES86}.
The {\em signature} of the language $\cal L$ of \CTL\ includes:
\begin{itemize}
  \item a finite set of Boolean variables, called {\em atoms} of $\cal L$: $\cal A$;
  \item constant symbols: $\bot$ and $\top$;
  \item the classical connectives: $\lor$ and $\neg$;
  \item the path quantifiers: $\ALL$ and $\EXIST$;
  \item the temporal operators: \NEXT, \FUTURE, \GLOBAL\ and \UNTIL, that
  means `neXt state', `some Future state', `all future states (Globally)' and `Until', respectively;
  \item parentheses: ( and ).
\end{itemize}

The priorities for the \CTL\ connectives are assumed to be in order as follows:
\begin{equation*}
  \neg, \EXIST\NEXT, \EXIST\FUTURE, \EXIST\GLOBAL, \ALL\NEXT, \ALL\FUTURE, \ALL\GLOBAL
 ,\land, \lor, \EXIST\UNTIL, \ALL\UNTIL, \rto,
\end{equation*}
where the leftmost (rightmost) symbol has the highest (lowest) priority.
Then the {\em existential normal form (or ENF in short) formulas} of
$\cal L$ are inductively defined via a Backus Naur form:
\begin{equation}\label{def:CTL:formulas}
  \phi ::=  \bot \mid \top \mid p \mid\neg\phi \mid \phi\lor\phi \mid
    \EXIST \NEXT \phi \mid
    \EXIST \GLOBAL \phi \mid
    \EXIST (\phi\ \UNTIL\ \phi)
\end{equation}
where $p\in\cal A$. The formulas $\phi\land\psi$ and $\phi\rto\psi$
are defined in a standard manner of propositional logic.
The other form formulas of $\cal L$ are abbreviated
using the forms of (\ref{def:CTL:formulas}).


Throughout this article we shall assume that every formula of $\cal L$ has bounded size, where
the size $|\varphi|$ of formula $\varphi$ is its length over the alphabet of $\cal L$~\cite{DBLP:journals/jcss/EmersonH85}.
As we will see later, this constraint will enable us to express the result of forgetting in \CTL\  in the form of a (disjunctive) \CTL\ formula. A  {\em theory} of $\cal L$ is a finite set of formulas of $\cal L$. By abusing the notation, we identify a theory $\Pi$ as the formula $\bigwedge\Pi$ whenever the context is clear.

We are now in the position to recall the semantics of $\cal L$.
Let ${\cal M}=(S,R,L,s_0)$ be an initial structure, $s\in S$ and $\phi$ a formula of $\cal L$.
The {\em satisfiability} relation between $({\cal M},s)$ and $\phi$,
written $({\cal M},s)\models\phi$, is 
defined
as follows:

\begin{itemize}
  \item $({\cal M},s)\not\models\bot$ \ and\  $({\cal M},s)\models\top$;
  \item $({\cal M},s)\models p$ iff $p\in L(s)$;
  \item $({\cal M},s)\models \phi_1\lor\phi_2$ iff
    $({\cal M},s)\models \phi_1$ or $({\cal M},s)\models \phi_2$;
  \item $({\cal M},s)\models \neg\phi$ iff  $({\cal M},s)\not\models\phi$;
  \item $({\cal M},s)\models \EXIST\NEXT\phi$ iff
    $({\cal M},s_1)\models\phi$ for some $(s,s_1)\in R$;
  \item $({\cal M},s)\models \EXIST\GLOBAL\phi$ iff
    $\cal M$ has a path $(s_1=s,s_2,\ldots)$ such that
    $({\cal M},s_i)\models\phi$ for each $i\ge 1$;
  \item $({\cal M},s)\models \EXIST(\phi_1\UNTIL\phi_2)$ iff
    $\cal M$ has a path $(s_1=s,s_2,\ldots)$ such that, for some $i\ge 1$,
    $({\cal M},s_i)\models\phi_2$ and
    $({\cal M},s_j)\models\phi_1$ for each $j~(1\leq j<i)$.
\end{itemize}

Similar to the work in \cite{browne1988characterizing,Bolotov:1999:JETAI},
only initial \MPK-structures are considered to be candidate models
in the following, unless otherwise noted. Formally,
an initial \MPK-structure $\cal K$ is a {\em model} of a formula $\phi$
whenever ${\cal K}\models\phi$.
We denote $\Mod(\phi)$ the set of models of $\phi$.
The formula
$\phi$  is {\em satisfiable}
if $\Mod(\phi)\neq\emptyset$.
Given two formulas $\phi_1$ and $\phi_2$,  by $\phi_1\models\phi_2$ we mean $\Mod(\phi_1)\subseteq\Mod(\phi_2)$, by $\phi_1\equiv\phi_2$ we mean $\phi_1\models\phi_2$ and $\phi_2\models\phi_1$.
In this case, $\phi_1$ is {\em equivalent} to $\phi_2$.
The set of atoms occurring in $\phi_1$ is denoted by $\Var(\phi_1)$.
The formula $\phi_1$ is {\em irrelevant to} the atoms in a set $V$ (or simply $V$-{\em irrelevant}), written $\IR(\phi_1,V)$,
if there is a formula $\psi$ with
$\Var(\psi)\cap V=\emptyset$ such that $\phi_1\equiv\psi$.

\section{Forgetting in \CTL}
\label{forgetting}
In this section, we present the notion of forgetting in \CTL\ and report its properties.
First, we give a general definition of \emph{bisimulation} between $\MPK$-structures, called $V$-bisimulation, to define forgetting in \CTL.
The notion of bisimulation captures the idea that the computation trees of two structures are behaviourally same.

 Second, the characterizing formula of an initial $\MPK$-structure on some set $V$ of propositions will be given. Then we will show that each initial $\MPK$-structure can be captured by a \CTL\ formula, and hence the result of forgetting $V$ from formula $\varphi$ can be expressed as a disjunction of the characterizing formulas of initial $\MPK$-structures which are $V$-bisimilar with some models of $\varphi$.
And last, the related properties, which include representation theorem, algebraic properties (i.e., Modularity, Commutativity and Homogeneity) of the forgetting operator,  and the complexity results on the fragment $\CTL_{\ALL\FUTURE}$, will be explored.

\subsection{$V$-bisimulation}

In our framework, we  will need to express bisimulation w.r.t. different sets of atomic variables explicitly under a single setting. Therefore, in this subsection, we define the notion of  $V$-bisimulation $\mathcal{B}^V$ which is a \emph{bisimulation w.r.t. a set $V$ of atomic propositions}. Hence, we also call it a set-based bisimulation.

In order to introduce  the actual notion,  we start with  the construction of $V$-bisimulation up to a certain degree (of depth) $n \in \mathbb{N}$  in the computation trees (denoted by $\Hb^V_n$)  which we will introduce next:


Let $V \subseteq \Ha$ and ${\cal K}_i=({\cal M}_i,s_i)$ with $i\in\{1,2\}$ and $\Hm_i=(S_i, R_i,L_i, s_0^i)$. 
\begin{itemize}
  \item $({\cal K}_1,{\cal K}_2)\in\Hb_0^V$ if $L_1(s_1)- V=L_2(s_2)- V$;  
  \item for $n\ge 0$, $({\cal K}_1,{\cal K}_2)\in\Hb_{n+1}^V$ if:
  \begin{itemize}
    \item $({\cal K}_1,{\cal K}_2)\in\Hb_0^V$,
    \item for every $(s_1,s_1')\in R_1$, there is a $(s_2,s_2')\in R_2$
    such that $({\cal K}_1',{\cal K}_2')\in \Hb_n^V$, and
    \item for every $(s_2,s_2')\in R_2$, there is a $(s_1,s_1')\in R_1$
    such that $({\cal K}_1',{\cal K}_2')\in \Hb_n^V$,
  \end{itemize}
  where ${\cal K}_i'=({\cal M}_i,s_i')$ with $i\in\{1,2\}$, and $n\in \mathbb{N}$.
\end{itemize}

In the rest of the paper, by bisimulation, we shall only refer to $V$-bisimulation. So to ease the notation, from now on we will omit the superscript $V$  in $\Hb_i^V$ and write $\Hb_i$ instead.

Now, we are ready to define the notion of $V$-bisimulation between \MPK-structures.
\begin{definition}[$V$-bisimulation]
  \label{def:V-bisimulation}
   Let $V\subseteq\cal A$. Given  two \MPK-structures ${\cal K}_1$ and ${\cal K}_2$ are $V$-{\em bisimilar},  denoted ${\cal K}_1 \lrto_V {\cal K}_2$,
 if and only if $ ({\cal K}_1,{\cal K}_2)\in {\Hb_n}\mbox{ for all } n\ge 0.$ Moreover, let $i\in \{1,2\}$, then two paths $\pi_i=(s_{i,1},s_{i,2},\ldots)$ of $\Hm_i$
 are $V$-{\em bisimilar} if
$ {\cal K}_{1,j} \lrto_V {\cal K}_{2,j}\mbox { for every $j \in  \mathbb{N}_{\geq 1}$ }$
 where ${\cal K}_{i,j}=(\Hm_i,s_{i,j})$.
\end{definition}

On the one hand,  this notion can be considered as a simple generalization of the classical
bisimulation-equivalence of Definition~7.1 in \cite{Baier:PMC:2008} when $V=\cal A$ and there is only one initial state (as in our case).

On the other hand, our definition of ${\Hb_n}$ is similar to
the state equivalence (i.e., $E_n$) in \cite{browne1988characterizing}, yet it is
different in the sense that ours is defined on \MPK-structures,
while state-equivalence is defined on states.
Moreover, our notion is also different
from  the state-based bisimulation notion of Definition~7.7 in \cite{Baier:PMC:2008},
which is defined for states of a given \MPK-structure.\footnote{As reported to us by an anonymous reviewer, there is also a notion of $k$-bisimulation~\cite{kaushik2002updates} outside the realm of logic (but from database literature), which has a similar intuition to our $\mathcal{B}_n$, yet in the opposite direction: they consider bisimilarity through parents of a node (states), while we consider successors in relations. Again our notion is defined over \MPK-structures. }

\begin{example}[cont'd from Example~\ref{car_manufacturing}]\label{ex:2}
Let us call the model given in the previous example as ${\cal K}_1$ with initial state $s_0$, \ie ${\cal K}_1=((S,R,L,s_0),s_0)$, as illustrated in Figure~\ref{v1uv2}. Then, ${\cal K}_2$ is obtained from ${\cal K}_1$ by  removing  $sp$,\footnote{It removes $sp$ from $L(s)$ for every $s\in S$. Note that $L(s_4)-\{sp\}=L(s_2)$.} and  ${\cal K}_3$ is obtained from ${\cal K}_2$ by removing $se$.
Observe that ${\cal K}_1\lrto_{\{sp\}} {\cal K}_2$, ${\cal K}_2\lrto_{\{se\}} {\cal K}_3$ and ${\cal K}_1\lrto_{\{sp,se\}} {\cal K}_3$. Besides,  ${\cal K}_1$ is not bisimilar~\cite{Baier:PMC:2008} with either ${\cal K}_2$ or ${\cal K}_3$.
\begin{figure}[ht]
  \centering
  \includegraphics[width=8cm]{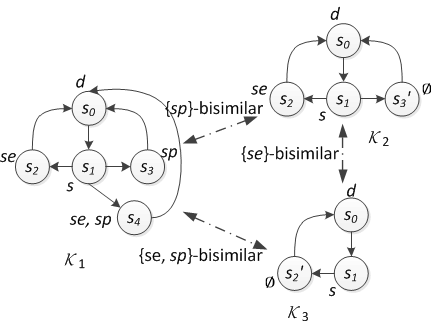}\\
  \caption{$V$-bisimulation between $\MPK$-structures}\label{v1uv2}
\end{figure}

\end{example}


 In the sequel, we shall simplify the notation further and write $s_1 \lrto_V s_2 $ to denote ${\cal K}_1 \lrto_V {\cal K}_2$ whenever the underlying initial structures are clear from the context.

\begin{lemma}\label{lem:equive}
  The relation $\lrto_V$ is an equivalence relation.
\end{lemma}

Next, we give some further key properties of $\lrto_V$ w.r.t. different $V$s.
\begin{proposition}\label{div}
Let $i\in \{1,2\}$, $V_1,V_2\subseteq\cal A$, $s_1'$ and $s_2'$ be two states,
  $\pi_1'$ and $\pi_2'$ be two paths,
and ${\cal K}_i=({\cal M}_i,s_i)~(i=1,2,3)$ be \MPK-structures
 such that
${\cal K}_1\lrto_{V_1}{\cal K}_2$ and ${\cal K}_2\lrto_{V_2}{\cal K}_3$.
 Then:
 \begin{enumerate}[(i)]
   \item $s_1'\lrto_{V_i}s_2'~(i=1,2)$ implies $s_1'\lrto_{V_1\cup V_2}s_2'$;
   \item $\pi_1'\lrto_{V_i}\pi_2'~(i=1,2)$ implies $\pi_1'\lrto_{V_1\cup V_2}\pi_2'$;
   \item for each path $\pi_{s_1}$ of $\Hm_1$ there is a path $\pi_{s_2}$  of $\Hm_2$ such that $\pi_{s_1} \lrto_{V_1} \pi_{s_2}$, and vice versa;
   \item ${\cal K}_1\lrto_{V_1\cup V_2}{\cal K}_3$;
   \item If $V_1 \subseteq V_2$ then ${\cal K}_1 \lrto_{V_2} {\cal K}_2$.
 \end{enumerate}
\end{proposition}

In Proposition~\ref{div}, properties $(i)$ to $(iii)$ are the standard properties for $V$-bisimulation.
Property $(iv)$ shows that if a \MPK-structure is $V_1$ and $V_2$-bisimilar with the other two \MPK-structures, respectively, then those two \MPK-structures are $V_1 \cup V_2$-bisimilar. For an example, see Figure~\ref{v1uv2}. This property is crucial for forgetting.
And last, $(v)$ says that if two \MPK-structures are  $V_1$-bisimilar, then they are $V_2$-bisimilar for any $V_2$ with $V_1\subseteq V_2 \subseteq \Ha$.

Intuitively, if two \MPK-structures are $V$-bisimilar, then they satisfy the same formula $\varphi$ that does not contain any atoms in $V$, i.e., $\IR(\varphi, V)$. This idea has been formalized and shown in the following theorem.

\begin{theorem}\label{thm:V-bisimulation:EQ}
  Let $V\subseteq\cal A$, ${\cal K}_i~(i=1,2)$ be two \MPK-structures such that
  ${\cal K}_1\lrto_V{\cal K}_2$ and $\phi$ be a formula with $\IR(\phi,V)$. Then
  ${\cal K}_1\models\phi$ if and only if ${\cal K}_2\models\phi$.
\end{theorem}


Below, we illustrate this idea over an example.

\begin{example}[cont'd from Example~\ref{ex:2}]\label{ex:3}
Let $\varphi_1= d \wedge \EXIST\FUTURE se \wedge \ALL\GLOBAL(se \rto \ALL\NEXT d)$ and $\varphi_2= d \wedge \ALL \NEXT se$ be two \CTL\ formulae.  They are  $\{sp\}$-irrelevant. One can see that ${\cal K}_1$ and ${\cal K}_2$ in Figure~\ref{v1uv2} satisfy $\varphi_1$, but not $\varphi_2$.

\end{example}
Next, we define the $V$-bisimulation between computation trees (of two initial structures). This construction will become useful when we define the characterizing formula of an initial $\MPK$-structure using the characterizing formula of a computation tree.

Let $V\subseteq\cal A$, ${\cal M}_i~(i=1,2)$ be  initial  structures.
A computation tree $\Tr_n(s_1)$ of ${\cal M}_1$ is $V$-{\em bisimilar}
to a computation tree $\Tr_n(s_2)$ of ${\cal M}_2$, written
$({\cal M}_1,\Tr_n(s_1))\lrto_V({\cal M}_2,\Tr_n(s_2))$ (or simply
$\Tr_n(s_1)\lrto_V\Tr_n(s_2)$), if 
\begin{itemize}
  \item $L_1(s_1)- V=L_2(s_2)- V$,
  \item For every subtree $\Tr_{n-1}(s_i')$ of $\Tr_n(s_i)$,
  $\Tr_n(s_{(i \mod 2)+1})$ has a subtree $\Tr_{n-1}(s_{(i \mod 2)+1}')$ such that
  $\Tr_{n-1}(s_i')\lrto_V\Tr_{n-1}(s_{(i \mod 2)+1}')$.
\end{itemize}
The last condition in the above definition
hold trivially for $n=0$.

\begin{proposition}\label{B_to_T}
  Let $V\subseteq\cal A$ and $({\cal M}_i,s_i)~(i=1,2)$ be two \MPK-structures.
  Then
  \[(s_1,s_2)\in{\cal B}_n\mbox{ iff }
  \Tr_j(s_1)\lrto_V\Tr_j(s_2)\mbox{ for every $0\le j\le n$}.\]
\end{proposition}

Proposition~\ref{B_to_T} says that a state $s_1$ of an initial structure is $V$-bisimilar to a state $s_2$ of another initial structure at a particular depth $n$ if, and only if,  all of the respective sub-trees rooted at $s_1$ and $s_2$ until depth $n$ are $V$-bisimilar.

Moreover, if two states $s$ and $s'$ from the same initial structure are not $V$-bisimilar, then the computation trees rooted at $s$ and $s'$, respectively, are not $V$-bisimilar at some depth  $k\in \mathbb{N}$. This is shown in the following proposition.

\begin{proposition}\label{pro:k}
  Let $V\subseteq \Ha$, $\Hm$ be an initial  structure and $s,s'\in S$
  such that $s\not\lrto_V s'$.
  There exists a least $k$ such that
  $\Tr_k(s)$ and $\Tr_k(s')$ are not $V$-bisimilar.
\end{proposition}

\subsection{Characterization of an Initial \MPK-structure}

In the following, we present characterizing formulas of initial \MPK-structures  over a signature to characterize the $\lrto_V$-class of an initial \MPK-structure.
\footnote{Similar approaches has been taken in the literature e.g., in~\cite{DBLP:conf/birthday/1997ehrenfeucht},  a class (namely, $\equiv_{\overline{k}}$-class) of structures of monadic formulas has been characterized by Hintikka formulae~\cite{hintikka1953distributive}. Another example is Yankov-Fine construction in \cite{yankov1968three}.}

To start with, we give the definition of characterizing formulas of computation trees.
\begin{definition}\label{def:V:char:formula}
Let $V\subseteq \Ha$, $\Hm =(S,R,L,s_0)$ be an initial structure and $s\in S$.
The {\em characterizing formula} of the computation tree $\Tr_n(s)$ on $V$,
written ${\cal F}_V(\Tr_n(s))$, is defined recursively as:
\begin{align*}
   {\cal F}_V(\Tr_0(s)) &=  \bigwedge_{p \in V\cap L(s)}p
     \wedge \bigwedge_{q\in V-L(s)} \neg q,\\
   {\cal F}_V(\Tr_{k+1}(s))& = \bigwedge_{(s,s')\in R}
    \EXIST \NEXT {\cal F}_V(\Tr_k(s'))\\
  \wedge &
    \ALL \NEXT \left( \bigvee_{(s,s')\in R} {\cal F}_V(\Tr_k(s')) \right) \wedge {\cal F}_V(\Tr_0(s))
\end{align*}
for $k\ge 0$.
\end{definition}
The characterizing formula of a computation tree formally exhibits the content of each node in $V$ (i.e., atoms in $V$ that are {\em true}  if they are in the label of this node of the computation tree, and {\em false} otherwise) and the temporal relation between states recursively.
Clearly, ${\cal F}_V(\Tr_0(s))$ expresses the content of node $s$ in terms of $V$, the conjunction with $\EXIST \NEXT$ part guarantees that each direct successor $s'$ of $s$ is captured by a \CTL\ formula until depth $k$, and the $\ALL \NEXT$ part guarantees that for each direct successor $s'$ of $s$ there exists another direct successor $s''$ of $s$ such that $s''$ is $V$-bisimilar to $s'$ until depth $k$.

The following result shows that the $V$-bisimulation between two computation trees implies the semantic equivalence of the corresponding characterizing formulas.

\begin{lemma}\label{lem:Vb:TrFormula:Equ}
Let $V\subseteq \Ha$, and $\Hm, \Hm'$ be two initial structures,
$s\in S$, $s'\in S'$ and $n\ge 0$. If $\Tr_n(s) \lrto_{\overline V} \Tr_n(s')$, then ${\cal F}_V(\Tr_n(s)) \equiv {\cal F}_V(\Tr_n(s'))$.
\end{lemma}


In Lemma~\ref{lem:Vb:TrFormula:Equ}, let $s'=s$. Then, it is easy to see that for any formula $\varphi$ of $V$, if $\varphi$ is a characterizing formula of $\Tr_n(s)$ then $\varphi \equiv {\cal F}_V(\Tr_n(s))$.

The notion of $V$-bisimulation and Proposition~\ref{pro:k} naturally induce a complementary notion, so-called $V$-\emph{distinguishability}, which will turn out to be useful in defining the characterizing formula of an initial \MPK-structure.
In particular, we will say that two states $s$ and $s'$ of $\Hm$ in Proposition~\ref{pro:k} are $V$-{\em distinguishable} if $s \not \lrto_{\overline V} s'$, and write that $\dis_V({\cal M},s,s',k)$, where we assume $k$ to be the  smallest natural number which makes $s$ and $s'$ $V$-distinguishable. Furthermore, we say that an initial  structure ${\cal M}$ is $V$-distinguishable if there are two states $s$ and $s'$ in $\Hm$ that are $V$-distinguishable. Then given an initial structure $\mathcal{M}$ and a set $V$ of atoms, the smallest value of $k$ which ensures $V$-distinguishability  is in question. We shall call such a $k$ as the \emph{characterization number} of $\mathcal{M}$ w.r.t. $V$ and define it formally as
\[ch({\cal M},V)=
\left\{
  \begin{array}{ll}
    \max\{k\mid s,s'\in S \text{ and }\dis_V({\cal M},s,s',k)\},\\
         \ \ \qquad \qquad \qquad \hbox{${\cal M}$ is $V$-distinguishable;} \\
    \min\{k\mid {\cal B}_{k}={\cal B}_{k+1}\}, \ \ \ \quad \qquad \hbox{otherwise.}
  \end{array}
\right.
\]
 since it will be crucial in defining the characterization formula (for a given initial $\MPK$-structure).

Observe that the $ch(\Hm, V)$ always exists for every initial  structure $\Hm$ and $V\subseteq \Ha$: If there are two states $s_1$ and $s_2$ such that $s_1$ and $s_2$ are $V$-distinguishable, then the characterization number exists by definition. In the extreme case, if for all $s, s'$ in $\Hm$, $((\Hm,s), (\Hm, s')) \in \Hb_k$ for all $k \geq 0$, and  $\Hb_k = \Hb_{k+1}$ (since the set of states in $\Hm$ is always finite), then the characterization number is 0.



Intuitively, given a state $s \in S$ of $\Hm$, the characterization number $c$ of $\Hm$ divides the states in $\Hm$ into two classes: The one which contains those states $s'$ until depth $c$ such that $(\Hm,s') \models {\cal F}_V(\Tr_c(s))$, and the other which contains the remaining states. Now, we are finally ready to define the characterizing formula of an  initial \MPK-structure.
\begin{definition}[Characterizing Formula]
Let $V\subseteq\cal A$,
 and ${\cal K}=({\cal M},s_0)$ be an initial \MPK-structure with $c=ch({\cal M},V)$, and for every state $s' \in S$ of $\Hm$, $T(s') = {\cal F}_V(\Tr_c(s'))$.
Then, the {\em characterizing formula} ${\cal F}_V({\cal K})$ of $\cal K$ on $V$ is:
\begin{align*}
  &T(s_0) \text{ } \wedge \\
  & \bigwedge_{s\in S}\ALL \GLOBAL\left(
    T(s) \rto
    \bigwedge_{(s,s')\in R}
        \EXIST \NEXT T(s')
        \wedge
        \ALL \NEXT (\bigvee_{(s,s')\in R}T(s'))
    \right)
\end{align*}
\end{definition}
Here, $T(s_0)$ ensures that the \MPK-structure starts from the initial state, and the remaining part ensures that we go deep enough in the computation tree (i.e., through all possible transitions from every state $s \in S$)  to detect any two  $V$-distinguishable states $s$ and $s'$ (which would then imply  $T(s) \not \equiv T(s')$). As a remark on notation, sometimes  we shall need to express the initial structure and the initial state explicitly, then we will use the rather transparent notation i.e., ${\cal F}_V(\Hm, s_0)$ (instead of ${\cal F}_V({\cal K})$).  

One can observe that $\IR({\cal F}_V(\Hm, s_0), \overline V)$.
Besides, given a set of atomic propositions $V$, any initial \MPK-structure has its own unique characterizing formula on $V$. As we will see later, the characterizing formula will play a crucial role in showing important properties of forgetting, as well as in our main contribution which is  computing the SNC and WSC of a \CTL\ formula under an initial \MPK-structure.


The following example illustrates how one can compute a characterizing formula:
\begin{example}[cont'd from Example~\ref{ex:2}]\label{ex:4}

Reconsider the ${\cal K}_2= (\Hm, s_0)$  in Figure~\ref{fig:K2Tree}, illustrated on the left side   (originally introduced in Figure~\ref{v1uv2}).   The corresponding computation trees are listed on the right side: from left to right, they are rooted at $s_0$ with depth $0$, $1$, $2$ and $3$, respectively. For simplicity, the labels of the nodes in the trees are omitted (See Figure~\ref{v1uv2} for the actual labels). Let $V=\{d\}$ then $\overline{V}=\{s, se\}$.

 We can see that $\Tr_0(s_1) \lrto_{\overline{V}} \Tr_0(s_2)$, since $L(s_1) - \overline{V} = L(s_2) - \overline{V}$. Moreover, $\Tr_1(s_1) \not \lrto_{\overline{V}} \Tr_1(s_2)$, since there is $(s_1, s_2)\in R$ such that for any $(s_2, s') \in R$, it is the case that $L(s_2)- \overline V \neq L(s') - \overline V$ (because there is only one direct successor $s'=s_0$). Hence, we have $s_1$ and $s_2$ which are $V$-distinguishable and  $\dis_{V}(\Hm, s_1, s_2, 1)$. Similarly, we have  $\dis_{ V}(\Hm, s_0, s_1, 0)$, $\dis_{V}(\Hm, s_0, s_2, 0)$ and $\dis_{ V}(\Hm, s_0, s_3', 0)$. Furthermore, we can see that $s_2 \lrto_{\overline V} s_3'$. Therefore, $ch(\Hm, V)=\max\{k\mid s,s'\in S \text{ and } \dis_{V}({\cal M},s,s',k)\} = 1$.
 And  we have the following:
 \begin{align*}
   {\cal F}_V(\Tr_0(s_0)) &= d, \qquad \quad {\cal F}_V(\Tr_0(s_1)) = \neg d, \\
   {\cal F}_V(\Tr_0(s_2)) &= \neg d,  \qquad  {\cal F}_V(\Tr_0(s_3')) = \neg d,\\
   {\cal F}_V(\Tr_1(s_0)) &= \EXIST\NEXT \neg d \wedge \ALL\NEXT \neg d \wedge d \equiv \ALL\NEXT \neg d \wedge d, \\
   {\cal F}_V(\Tr_1(s_1)) &= \EXIST\NEXT \neg d \wedge \EXIST\NEXT \neg d  \wedge \ALL\NEXT (\neg d \vee \neg d) \wedge \neg d \\
   &\equiv \ALL\NEXT \neg d \wedge \neg d, \\
   {\cal F}_V(\Tr_1(s_2)) &= \EXIST\NEXT d  \wedge \ALL\NEXT d \wedge \neg d \equiv \ALL\NEXT d \wedge \neg d,\\
   {\cal F}_V(\Tr_1(s_3')) &\equiv {\cal F}_V(\Tr_1(s_2)),\\
  {\cal F}_V(\Hm, s_0)&\equiv \ALL\NEXT \neg d \wedge d \wedge \\
  & \ALL \GLOBAL(\ALL\NEXT \neg d \wedge d \rto \ALL\NEXT(\ALL\NEXT \neg d \wedge \neg d))\wedge \\
  & \ALL \GLOBAL(\ALL\NEXT \neg d \wedge \neg d \rto \ALL\NEXT(\ALL\NEXT d \wedge \neg d)) \wedge\\
  & \ALL \GLOBAL(\ALL\NEXT d \wedge \neg d \rto \ALL\NEXT(\ALL\NEXT \neg d \wedge d)).
\end{align*}

 \begin{figure}
  \centering
  \includegraphics[width=8.5cm]{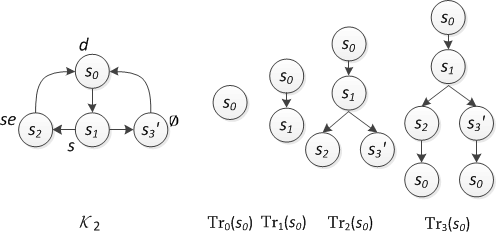}
  \caption{On the left side,  $\mathcal{K}_2$ (aforementioned in Figure~~\ref{v1uv2}), and on the right side, the corresponding computation trees of depth 0, 1, 2 and 3, respectively. Labels of the nodes are omitted for simplicity.}\label{fig:K2Tree}
\end{figure}

\end{example}

The following result shows that there is a correspondence between the semantic equivalence of characterizing formulae and the initial \MPK-structures which are $V$-bisimilar. That is, two initial \MPK-structures are $V$-bisimilar if, and only if  their characterizing formulae are semantically equivalent. This means,  characterizing formula characterizes initial \MPK-structures which are equivalent up to $V$-bisimulation.


\begin{theorem}\label{CF}
Let $V\subseteq \Ha$, $\Hm=(S,R,L,s_0)$ and $\Hm'=(S',R', L',s_0')$ be two initial structures. Then,
\begin{enumerate}[(i)]
 \item $(\Hm',s_0') \models {\cal F}_V({\cal M},s_0)
\text{ iff }
({\cal M},s_0) \lrto_{\overline V} ({\cal M}',s_0')$;

\item $s_0 \lrto_{\overline V} s_0'$ implies  ${\cal F}_V(\Hm, s_0) \equiv {\cal F}_V(\Hm', s_0')$.
\end{enumerate}

\end{theorem}





It is noteworthy that under our assumption of bounded size (of a \CTL\ formula), say $n$, it will be sufficient to consider the models of formulas within a state space $S$ satisfying $|S|=n8^n$~\cite{DBLP:journals/jcss/EmersonH85}.
Any other model must be bisimilar to some model within the state space, and their characterizing formulas are equivalent due to Theorem~\ref{CF}. Therefore, given a formula of size within the bound, only a finite number of such initial \MPK-structures need to be considered as the candidate models. This fact is expressed in the following lemma.
\begin{lemma}\label{lem:models:formula}
  Let $\varphi$ be a formula. We have
  \begin{equation}
    \varphi\equiv \bigvee_{(\Hm, s_0)\in\Mod(\varphi)}{\cal F}_{\cal A}(\Hm, s_0).
\end{equation}
\end{lemma}
Yet Lemma~\ref{lem:models:formula} has an additional message: Any \CTL\ formula  can be expressed in the form of a disjunction of the characterizing formulae of its models. This fact will be  crucial  in the results we present in next sections.

\subsection{Semantic Properties of Forgetting in \CTL}
In this subsection, we present the notion of forgetting in \CTL\ and investigate its semantic properties. Let us start with the formal definition.
%
\begin{definition}[Forgetting]\label{def:V:forgetting}
  Let $V\subseteq\cal A$ and $\phi$ be a formula.
A formula $\psi$ with $\Var(\psi)\cap V=\emptyset$
is a {\em result of forgetting $V$ from} $\phi$ (denoted as $\CTLforget(\phi,V)$), if
\begin{equation*}
\resizebox{.91\linewidth}{!}{$
\displaystyle
  \Mod(\psi)=\{{\cal K}\mbox{ is initial}\mid \exists {\cal K}'\in\Mod(\phi)\ \text{ \st }\ {\cal K}'\lrto_V{\cal K}\}.
  $}
\end{equation*}
\end{definition}
Realize that Definition~\ref{def:V:forgetting} implies if both $\psi$ and $\psi'$ are results of forgetting $V$ from $\phi$, then
$\Mod(\psi)=\Mod(\psi')$, i.e., $\psi$ and $\psi'$ have the same models. In this sense, the result of  forgetting $V$ from $\phi$ is unique (up to semantic equivalence).
By Lemma~\ref{lem:models:formula}, such a formula always exists, which
is equivalent to
\begin{equation*}
  \bigvee_{{\cal K}\in  \{{\cal K}'\mid \exists {\cal K}''\in\Mod(\phi)\ \text{ and }\ {\cal K}''\lrto_V{\cal K}'\}} {\cal F}_{\overline V}({\cal K}).
\end{equation*}
At this point, it is important to emphasize that, the notion of forgetting  we have defined for \CTL\ respects the classical forgetting defined for propositional logic (PL)~\cite{lin1994forget}. To see this, assume that $\varphi$ is a PL formula and $p\in \Ha$, then $\Forget(\varphi, p)$ is a result of forgetting $p$ from $\varphi$; that is, $\Forget(\varphi, p)\equiv \varphi[p/\bot] \vee \varphi[p/\top]$.
That way, given a set $V\subseteq \Ha$, one can recursively define $\Forget(\varphi, V\cup \{p\}) = \Forget(\Forget(\varphi, p),V)$, where $\Forget(\varphi, \emptyset) = \varphi$. Using this insight, the following result shows that the classical notion of forgetting (for PL ~\cite{lin1994forget}) is a special case of forgetting in CTL.

\begin{theorem}\label{thm:PL:CTL}
Let $\varphi$ be a PL formula and $V\subseteq \Ha$, then
\[
\CTLforget(\varphi, V) \equiv \Forget(\varphi, V).
\]
\end{theorem}

In~\cite{Yan:AIJ:2009}, authors give four postulates concerning knowledge forgetting in  \SFive\ modal logic (also called \emph{forgetting postulates}) which can be  considered as desirable properties of such a notion. In the following, we first list these postulates, and then show that our notion of forgetting in \CTL\ satisfies them.

\textbf{Forgetting postulates}~\cite{Yan:AIJ:2009} are:
\begin{itemize}
  \item[] (\W) Weakening: $\varphi \models \varphi'$;
  \item[] (\PP) Positive Persistence:
  for any formula $\eta$, if $\IR(\eta, V)$ and $\varphi \models \eta$ then $\varphi' \models \eta$;
  \item[] (\NgP) Negative Persistence :  for any formula $\eta$,  if $\IR(\eta, V)$ and $\varphi \not \models \eta$ then $\varphi' \not \models \eta$;
  \item[] (\textbf{IR}) Irrelevance: $\IR(\varphi', V)$
\end{itemize}
where $V\subseteq\cal A$,
$\varphi$ is a formula and $\varphi'$ is a result of
forgetting $V$ from $\varphi$.
%
Intuitively, the postulate (\W) says, forgetting weakens the original formula; the postulates  (\PP) and (\NgP)
say that  forgetting results have no effect on formulas that are irrelevant to forgotten atoms; the postulate (\textbf{\IR}) states that
forgetting result is irrelevant to forgotten atoms.
It is noteworthy that they are not all orthogonal  e.g., (\NgP) is a consequence of (\W) and (\PP). Nonetheless, we prefer to list them all, in order to outline the basic intuition behind them.

\begin{theorem}[Representation Theorem]\label{thm:close}
Let $\varphi$ and $\varphi'$ be \CTL\ formulas and $V \subseteq \Ha$.
The following statements are equivalent:
\begin{enumerate}[(i)]
  \item $\varphi' \equiv \CTLforget(\varphi, V)$,
  \item $\varphi'\equiv \{\phi \mid\varphi \models \phi \text{ and } \IR(\phi, V)\}$,
  \item Postulates (\W), (\PP), (\NgP) and (\textbf{IR}) hold if $\varphi,   \varphi'$ and $V$ are as in (i) and (ii).
\end{enumerate}
\end{theorem}
\begin{proof}
$(i) \LRto (ii)$. To prove this, it is enough to show that:
\begin{align*}
 & \Mod(\CTLforget(\varphi, V)) = \Mod(\{\phi | \varphi \models \phi, \IR(\phi, V)\})\\
 & = \Mod(\bigvee_{\Hm, s_0\in \Mod(\varphi)} {\cal F}_{\Ha- V}(\Hm, s_0)).
\end{align*}
First, suppose that $(\Hm', s_0')$ is a model of $\CTLforget(\varphi, V)$. Then there exists an initial \MPK-structure $(\Hm, s_0)$ which is a model of $\varphi$ and $(\Hm, s_0) \lrto_V (\Hm', s_0')$. By Theorem~\ref{thm:V-bisimulation:EQ}, we have $(\Hm', s_0') \models \phi$ for all $\phi$ such that $\varphi\models \phi$ and $\IR(\phi, V)$. Thus, $(\Hm', s_0')$ is a model of the theory $\{\phi \mid \varphi \models \phi, \IR(\phi, V)\}$.

Second, suppose that $(\Hm', s_0')$ is a model of $\{\phi \mid \varphi \models \phi$, $\IR(\phi, V)\}$. Thus, $(\Hm', s_0')$ $\models$ $\bigvee_{(\Hm, s_0)\in \Mod(\varphi)} {\cal F}_{\Ha- V}(\Hm, s_0)$ since  $\bigvee_{(\Hm, s_0)\in \Mod(\varphi)} {\cal F}_{\Ha- V}(\Hm, s_0)$ is irrelevant to $V$ and $\varphi \models$ $\bigvee_{(\Hm, s_0)\in \Mod(\varphi)} {\cal F}_{\Ha- V}(\Hm, s_0)$ by Lemma~\ref{lem:models:formula}.

Last, suppose that $(\Hm', s_0')$ is a model of $\bigvee_{\Hm, s_0\in \Mod(\varphi)} {\cal F}_{\Ha- V}(\Hm, s_0)$. Then there exists $(\Hm, s_0) \in \Mod(\varphi)$ such that $(\Hm', s_0') \models {\cal F}_{\Ha- V}(\Hm, s_0)$. Hence, $(\Hm, s_0)$ $\lrto_V$ $(\Hm', s_0')$ by Theorem~\ref{CF}. Thus $(\Hm', s_0')$ is also a model of $\CTLforget(\varphi,V)$.

$(ii)\Rto (iii)$. This is rather straightforward, so we put it into the supplementary material.

$(iii)\Rto (ii)$. 
By Positive Persistence, we have $\varphi' \models \{\phi \mid \varphi \models \phi, \IR(\phi, V)\}$.
The  $\{\phi \mid \varphi \models \phi, \IR(\phi, V)\} \models \varphi'$ can be obtained from (\W) and (\textbf{IR}).
Thus, $\varphi'$ is equivalent to $\{\phi \mid \varphi \models \phi, \IR(\phi, V)\}$.
\end{proof}
%


It is noteworthy  that the postulate \textbf{IR} is of crucial importance for computing SNC and WSC. Consider the $\psi=\varphi\wedge (q\lrto \alpha)$. If $\IR(\varphi \wedge \alpha, \{q\})$, then the result of forgetting $q$ from $\psi$ is $\varphi$. This property is described in the following lemma, and as we will later see in Section~\ref{ns_conditions}, it will become important (in reducing the SNC (WSC) of any \CTL\ formula to the one of a proposition).

\begin{lemma}\label{lem:KF:eq}
	Let $\varphi$ and $\alpha$ be two \CTL\ formulae and $q\in
		\overline{\Var(\varphi) \cup \Var(\alpha)}$. Then
	$\CTLforget(\varphi \wedge (q\lrto\alpha), q)\equiv \varphi$.
\end{lemma}


In what follows, we list other interesting properties of the forgetting operator. According to the definition of forgetting, the set of atoms to be forgotten should be forgotten as a whole.
The following property guarantees that this can be achieved modularly by applying forgetting one by one to the atoms to be forgotten.
\begin{proposition}[Modularity]\label{disTF}
Given a formula $\varphi \in \CTL$, $V$ a set of atoms and $p$ an atom such that $p \notin V$. Then,
\[
\CTLforget(\varphi, \{p\} \cup V) \equiv \CTLforget(\CTLforget(\varphi, p), V).
\]
\end{proposition}
The next property follows from the above proposition.

\begin{corollary}[Commutativity]\label{disTFV}
Let $\varphi$ be a formula and $V_i\subseteq{\cal A}~(i=1,2)$. Then:
\[
\CTLforget(\varphi, V_1 \cup V_2) \equiv \CTLforget(\CTLforget(\varphi, V_1), V_2).
\]
\end{corollary}

The following properties show that the forgetting respects the basic semantic notions of logic. They hold in both classical propositional logic and modal logic \SFive~\cite{Yan:AIJ:2009}. Below we show that they are also satisfied in our notion forgetting in \CTL.
\begin{proposition}\label{pro:ctl:forget:1}
Let $\varphi$, $\varphi_i$, $\psi_i$ ($i=1,2$) be formulas in \CTL\ and $V\subseteq \Ha$. We have
\begin{enumerate}[(i)]
  \item $\CTLforget(\varphi, V)$ is satisfiable iff $\varphi$ is;
  \item If $\varphi_1 \equiv \varphi_2$, then $\CTLforget(\varphi_1, V) \equiv \CTLforget(\varphi_2, V)$;
  \item If $\varphi_1 \models \varphi_2$, then $\CTLforget(\varphi_1, V) \models \CTLforget(\varphi_2, V)$;
  \item $\CTLforget(\psi_1 \vee \psi_2, V) \equiv \CTLforget(\psi_1, V) \vee \CTLforget(\psi_2, V)$;
  \item $\CTLforget(\psi_1 \wedge \psi_2, V) \models \CTLforget(\psi_1, V) \wedge \CTLforget(\psi_2, V)$;
\end{enumerate}
\end{proposition}

The next property shows that forgetting a set $V\subseteq\cal A$ from a formula with path quantifiers is equivalent to quantify the result of forgetting $V$ from the formula with the same path quantifiers.
\begin{proposition}[Homogeneity]\label{pro:ctl:forget:2}
  Let $V\subseteq\cal A$ and $\phi \in \CTL$,
  \begin{enumerate}[(i)]
    \item $\CTLforget(\ALL\NEXT\phi,V)\equiv \ALL\NEXT \CTLforget(\phi,V)$.
    \item $\CTLforget(\EXIST\NEXT\phi,V)\equiv\EXIST\NEXT \CTLforget(\phi,V)$.
    \item $\CTLforget(\ALL \FUTURE\phi,V)\equiv \ALL \FUTURE \CTLforget(\phi,V)$.
    \item $\CTLforget(\EXIST\FUTURE\phi,V)\equiv\EXIST\FUTURE \CTLforget(\phi,V)$.
  \end{enumerate}
\end{proposition}



\subsection{Complexity Results}
In the following, we analyze the computational complexity of the various tasks regarding the forgetting in the fragment $\CTL_{\ALL\FUTURE}$.
The fragment $\CTL_{\ALL\FUTURE}$ of \CTL, in which each formula contains only $\ALL \FUTURE$ temporal connective, corresponds to specifications that are expected to hold in all branches eventually. Such properties are of special interest in concurrent systems e.g., mutual exclusion and  waiting events~\cite{Baier:PMC:2008}. Our first result shows that the problem of model checking for forgetting of $V$ from $\varphi$ is $\textsc{NP}$-complete, if $\varphi \in \CTL_{\ALL\FUTURE}$.
\begin{proposition}[Model Checking]
	\label{modelChecking}
Given an initial \MPK-structure $(\Hm,s_0)$, $V\subseteq{\cal A}$ and $\varphi \in \CTL_{\ALL\FUTURE}$,  deciding $(\Hm,s_0) \models^? \CTLforget(\varphi, V)$ is \textsc{NP}-complete.
\end{proposition}

In the following, we investigate some complexity results concerning forgetting and the logical entailment in this fragment.


\begin{theorem}[Entailment]
	\label{thm:comF}
Let $\varphi$ and $\psi$ be two $\CTL_{\ALL \FUTURE}$ formulas and $V$ be a set of atoms. Then,
\begin{enumerate}[(i)]
  \item deciding  $\CTLforget(\varphi, V ) \models^? \psi$ is co-$\textsc{NP}$-complete,
  \item deciding  $\psi \models^? \CTLforget(\varphi, V)$ is $\Pi_2^{\textsc{P}}$-complete,
  \item deciding $\CTLforget(\varphi, V) \models^? \CTLforget(\psi, V)$ is $\Pi_2^{\textsc{P}}$-complete.
\end{enumerate}
\end{theorem}

\begin{proof}
(i) and (iii) is moved to supplementary material due to space restrictions.
(ii) Membership: We consider the complement of the
 problem. Guess an initial \MPK-structure $(\Hm, s_0)$ which has  polynomial size in the size of $\psi$ satisfying $\psi$ and check $(\Hm,s_0)$ $\not \models \CTLforget($ $\varphi$, $V)$. By Proposition~\ref{modelChecking}, it is in $\Sigma_2^{\textsc{P}}$. So the original problem is in $\Pi_2^{\textsc{P}}$. Hardness: Let $\psi \equiv \top$. Then the problem is reduced to decide the validity of  $\CTLforget(\varphi, V )$. Since propositional forgetting is a special case by Theorem~\ref{thm:PL:CTL}, the hardness  follows from the proof of Proposition 24 in~\cite{DBLP:journals/jair/LangLM03}.
\end{proof}

The following results are implications of Theorem~\ref{thm:comF}.
\begin{corollary}
Let $\varphi$ and $\psi$ be two $\CTL_{\ALL \FUTURE}$ formulas and $V$ a set of atoms. Then
\begin{enumerate}[(i)]
  \item deciding $\psi \equiv^?\CTLforget(\varphi, V)$ is $\Pi_2^{\textsc{P}}$-complete,
  \item deciding $\CTLforget(\varphi, V) \equiv^? \varphi$ is co-$\textsc{NP}$-complete,
  \item deciding $\CTLforget(\varphi, V) \equiv^? \CTLforget(\psi, V)$ is $\Pi_2^{\textsc{P}}$-complete.
\end{enumerate}
\end{corollary}

\section{Necessary and  Sufficient Conditions}
\label{ns_conditions}
In this section, we present the final key notions of our work:  namely, the \emph{strongest necessary condition} (SNC) and the \emph{weakest sufficient condition}  (WSC)  of a given \CTL\ specification.  As aforementioned in the introduction, these notions (introduced by E. Dijkstra in \cite{DBLP:journals/cacm/Dijkstra75}) correspond to the \emph{most general consequence} and the \emph{most specific abduction} of a specification, respectively, and have been central to a wide variety of tasks and studies (see Related Work). Our contribution, in particular, will be on computing SNC and WSC via forgetting under a given initial \MPK-structure and a set $V$ of atoms.  Let us give the formal definition. 
\begin{definition}[sufficient and necessary condition]\label{def:NC:SC}
Let $\phi$ be a formula (or an initial \MPK-structure), $\psi$ be a formula, $V \subseteq \Var(\phi)$, $q\in\Var(\phi)- V$
and $\Var(\psi)\subseteq V$.
\begin{itemize}
  \item $\psi$  is a {\em necessary condition} (NC in short) of $q$ on $V$ under $\phi$
    if $\phi \models q \rto \psi$.
  \item $\psi$  is a {\em sufficient condition} (SC in short) of $q$ on $V$ under $\phi$
    if $\phi \models \psi\rto q$.
  \item $\psi$  is a {\em strongest necessary condition} (SNC in short)
  of $q$ on $V$ under $\phi$
    if it is a NC of $q$ on $V$ under $\phi$, and $\phi\models\psi\rto\psi'$
    for any NC $\psi'$ of $q$ on $V$ under $\phi$.

    \item $\psi$  is a {\em weakest sufficient condition} (WSC in short)
  of $q$ on $V$ under $\phi$
    if it is a SC of $q$ on $V$ under $\phi$, and $\phi\models\psi'\rto\psi$
    for any SC $\psi'$ of $q$ on $V$ under $\phi$.
\end{itemize}
\end{definition}
Note that if both $\psi$ and $\psi'$ are SNC (WSC) of $q$ on $V$ under $\phi$, then
$\Mod(\psi)=\Mod(\psi')$, i.e., $\psi$ and $\psi'$ have the same models.
In this sense, the SNC (WSC) of $q$ on $V$ under $\phi$ is unique (up to semantic equivalence). The following result shows that the SNC and WSC are in fact dual notions.

\begin{proposition}[Dual]\label{dual}
 Let $V,q,\varphi$ and $\psi$ are defined as in Definition~\ref{def:NC:SC}.
 Then, $\psi$ is a SNC (WSC) of $q$ on $V$ under $\varphi$ iff $\neg \psi$ is a WSC (SNC)
    of $\neg q$ on $V$ under $\varphi$.
\end{proposition}

%

In order to generalise Definition~\ref{def:NC:SC} to arbitrary formulas, one can replace $q$ (in the definition)  by any formula $\alpha$, and redefine  $V$ as a subset of $\Var(\alpha) \cup \Var(\phi)$.
    It turns out that the previous notions of SNC and WSC for an atomic variable can be lifted to any formula, or, conversely, the SNC and WSC of any formula can be reduced to that of an atomic variable, as the following result shows.
\begin{proposition}\label{formulaNS_to_p}
     Let $\Gamma$ and $\alpha$ be two formulas, $V \subseteq \Var(\alpha) \cup \Var(\Gamma)$  and $q$ be a new proposition not in $\Gamma$ and $\alpha$.
 Then, a formula $\varphi$ of $V$ is the SNC (WSC) of $\alpha$ on $V$ under  $\Gamma$ iff it is the SNC (WSC) of $q$ on $V$ under $\Gamma' = \Gamma \cup \{q \lrto \alpha\}$.
 \end{proposition}

To give an intuition for WSC, we give the following example. The intuition for SNC is dual.

\begin{example}[cont'd from Example~\ref{ex:2}]\label{examp:WSC}
Recall ${\cal K}_2$ in Figure~\ref{v1uv2}. Let $\psi = \EXIST \NEXT(s \wedge (\EXIST \NEXT se \vee \EXIST \NEXT \neg d))$, $\varphi = \EXIST \NEXT(s \wedge \EXIST \NEXT \neg d)$, $\Ha =\{d, s, se\}$ and $V = \{s, d\}$, then we can check  that the WSC of $\psi$ on $V$ under ${\cal K}_2$ is $\varphi$.

We verify this result by the following two steps:
\begin{enumerate}[(i)]
  \item Observe that $\varphi \models \psi$ and $\Var(\varphi) \subseteq V$. Besides, $(\Hm, s_0) \models \varphi \wedge \psi$, hence ${\cal K}_2 \models \varphi \rto \psi$, which means $\varphi$ is a SC of $\psi$ on $V$ under ${\cal K}_2$,
  \item We will show that for any SC $\varphi'$ of $\psi$ on $V$ under ${\cal K}_2$,  we have ${\cal K}_2 \models \varphi' \rto \varphi$. It is easy to see that if ${\cal K}_2 \not \models \varphi'$, then ${\cal K}_2\models \varphi' \rto \varphi$, trivially. Now let's assume ${\cal K}_2 \models \varphi'$. In this case, we have $\varphi' \models \psi$ since $\varphi'$ is a SC of $\psi$ on $V$ under ${\cal K}_2$. Therefore, there is $\varphi' \models \EXIST \NEXT(s \wedge \phi)$, in which $\phi$ is a formula such that $\phi\models \EXIST \NEXT se \vee \EXIST \NEXT \neg d$. And then $\phi \models \EXIST \NEXT \neg d$ since $\IR(\varphi', \overline V)$. Hence, $\varphi' \models \varphi$ and we get  ${\cal K}_2 \models \varphi' \rto \varphi$, as desired.
\end{enumerate}
\end{example}

The following result establishes the bridge between forgetting and the notion of SNC (WSC) which are central to our contribution.


\begin{theorem}\label{thm:SNC:WSC:forget}
 Let $\varphi$ be a formula, $V\subseteq\Var(\varphi)$ and $q\in\Var(\varphi)- V$.
 \begin{enumerate}[(i)]
   \item $\CTLforget (\varphi \land q$, $(\Var(\varphi) \cup \{q\}) - V)$
   is a SNC of $q$ on $V$ under $\varphi$.
   \item  $\neg\CTLforget (\varphi \land \neg q$, $(\Var(\varphi) \cup \{q\}) - V)$
   is a WSC of $q$ on $V$ under $\varphi$.
 \end{enumerate}
 \end{theorem}

Following Theorem~\ref{thm:SNC:WSC:forget}, assume that $\beta = \CTLforget(\varphi \wedge q, (\Var(\varphi) \cup \{q\})- V)$.  Then, $\varphi \wedge q \models \beta$  by (\W). Moreover,  $\varphi \wedge q \models \beta$,  and then $\beta$ is a NC of $q$ on $V$ under $\varphi$.
In addition, for any $\psi$ with $\IR(\psi, (\Var(\varphi) \cup \{q\})- V)$ and $\varphi \wedge q \models \psi$, 
we have $\beta \models \psi$ by (\PP). Therefore, $\beta$ is the SNC of $q$ on $V$ under $\varphi$. This shows the intuition of how the SNC can be obtained from the forgetting.

Since any initial $\MPK$-structure can be characterized by a \CTL\ formula, by  Theorem~\ref{thm:SNC:WSC:forget} one can obtain the SNC (and its dual WSC) of a target property (a formula) under an initial $\MPK$-structure just by forgetting. This is shown in the following result.
\begin{theorem}\label{thm:inK:SNC}
Let ${\cal K}= (\Hm, s)$ be an initial \MPK-structure with $\Hm=(S,R,L,s_0)$ on the set $\Ha$ of atoms, $V \subseteq \Ha$ and $q\in V' = \Ha - V$. Then,
 \begin{enumerate}[(i)]
   \item the SNC of $q$ on $V$ under ${\cal K}$ is $\CTLforget({\cal F}_{\Ha}({\cal K}) \wedge q, V')$.
   \item the WSC of $q$ on $V$ under ${\cal K}$ is $\neg \CTLforget({\cal F}_{\Ha}({\cal K}) \wedge \neg q, V')$.
 \end{enumerate}
\end{theorem}

\section{An Algorithm for Forgetting in \CTL\ }
\label{section_algorithm}
The technical developments we have presented in previous sections naturally induce a procedure to compute forgetting in CTL. We think that it is useful to outline such a procedure explicitly in the form of an algorithm.  It is a model-based approach (presented in Algorithm~\ref{alg:compute:forgetting:by:VB}); that is,  it will compute the forgetting applied to a formula, simply  by considering all the possible models of that formula.  Its correctness is guaranteed by Lemma~\ref{lem:models:formula} and Theorem~\ref{CF}.

\begin{algorithm}[tb]
	\caption{\small A model-based \CTL\ forgetting procedure}
	\label{alg:compute:forgetting:by:VB}
	\KwIn{A \CTL formula $\varphi$ and a set $V$ of atoms}
	\KwOut{$\CTLforget(\varphi, V)$}
	$\psi \lto \bot$\;
	\ForEach{initial \MPK-structure $\cal K$ (over $\cal A$ and $\cal S$)}{
		\lIf{${\cal K}\not\models\varphi$}{{\bf continue}} 
		\ForEach{initial \MPK-structure ${\cal K}'$ with ${\cal K}\lrto_{ V}{\cal K}'$}{
			$\psi \lto \psi \lor {\cal F}_{\overline V}({\cal K}')$\;
		}
	}
	\Return $\psi$\;
\end{algorithm}

The example we give below echoes the initial example which was given in the introduction, and finalizes the running example with a simple intuition of forgetting.
\begin{example}\label{ex:6}
Recall the \MPK-structure ${\cal K}_1$  given in Figure~\ref{v1uv2}, and assume that we are given a property $\alpha = \EXIST\FUTURE(se \wedge sp)$. It is easy to see that ${\cal K}_1$ in Figure~\ref{v1uv2} satisfy $\alpha$. If $sp$ is intended to be removed, i.e., forgetting  $sp$ from $\alpha$,  then  $\CTLforget(\alpha,\{ sp\}) \equiv \EXIST\FUTURE se$. Hence, the company can use the new specification $\EXIST\FUTURE se$ to guide the new production process (which guarantees that the sedan car is eventually produced).
\end{example}



As we will show below, computing the forgetting by going through all the models is not very efficient, as one might expect. However, settling it is important from a theoretical point of view i.e., to see how costly is the naive approach.

\begin{proposition}\label{pro:time:alg1}
Let $\varphi$ be a \CTL\ formula and $V\subseteq \Ha$ with $|{\cal S}|=m$, $|\Ha|=n$ and $|V|=x$. Then the space complexity is $O((n-x)m^{2(m+2)}2^{nm} * \log m)$ and the time complexity of Algorithm~\ref{alg:compute:forgetting:by:VB} is at least the same as the space.
\end{proposition}

As expected, Algorithm~\ref{alg:compute:forgetting:by:VB} has a high cost; namely, \textsc{ExpSpace} complexity in the size of the state space and $\Ha$, which does not look encouraging. However, we believe that settling this result is important both from a theoretical and a practical point of view. Theoretically, it gives us a picture about the worst case, and urges us to come up with more efficient syntactical approaches which is a part of our future agenda. Moreover, we believe that model-based investigation and some of the structural observations we have made provide us with  informative valuable insights, which in turn could be useful in designing future algorithms which can exploit these observations, and potentially could lead to even efficient approximations with provably good bounds. Such future developments might prove important in developing practical algorithms as well.

\section{Concluding Remarks}
\label{conclusion}
\paragraph{Summary}
In this paper, we have presented the notion of forgetting for \CTL\
which enables computing weakest sufficient and strongest necessary conditions of specifications. In doing so, we introduced and employed the notion of $V$-bisimulation which can be considered as  a simple variable based generalisation of classical bisimulation. Furthermore, we have studied formal properties of forgetting, among them, homogeneity, modularity and commutativity. In particular, we have shown that our notion of forgetting satisfies the existing postulates of forgetting, which means it faithfully extends the notion of forgetting from classical propositional logic and modal logic \SFive\ to \CTL.
On the complexity theory side, we have investigated the model checking and the entailment problems of forgetting in the fragment $\CTL_{\ALL\FUTURE}$, which turn out to be $\textsc{NP}$-complete and range from co-$\textsc{NP}$ to $\Pi_2^\textsc{P}$-completeness, respectively.
And finally, we proposed a model-based algorithm which computes the forgetting of a given formula and a set of variables, and outlined its complexity.

\paragraph{Future work}
Note that, when a transition system $\cal M$ does not satisfy a specification $\phi$, one can evaluate the weakest sufficient condition  $\psi$ over a signature $V$ under which ${\cal M}$ satisfies $\phi$, viz., ${\cal M}\models\psi\rto \phi$ and $\psi$ mentions only atoms from $V$. It is worthwhile to explore how the condition $\psi$ can guide the design of a new transition system ${\cal M}'$ satisfying $\phi$.

Moreover, a further study regarding the computational complexity for other general fragments is required and part of the future research agenda. As mentioned in Section~\ref{section_algorithm}, these high complexity results are encouraging for other syntactic approaches e.g., proof-theoretic. Such investigation can be coupled with fine-grained parameterized analysis, as well as a search for approximation algorithms with provably good accuracy bounds.

\section*{Acknowledgements}
We kindly thank all the anonymous reviewers whose comments improved this work to a great extent.
Renyan Feng is funded by China Scholarship Council (CSC) grant number 201906670007.
Renyan Feng and Yisong Wang is supported by the National Natural Science Foundation of P.R. China under Grants 61976065, 61370161 and U1836105.
Erman Acar's research is funded by MaestroGraph research programme (NWO) with project number 612.001.552.

\bibliographystyle{kr}
\bibliography{ijcai20}

 \clearpage
 \appendix
 \section{Supplementary Material: Proof Appendix}
The results in the appendix follows the order in the text. Additional auxiliary lemmas and propositions in the appendix respect that order as well.\\

  \noindent\textbf{Section \ref{forgetting}   Forgetting in \CTL}\\

  \noindent\textbf{Section \ref{forgetting}.1  $V$-bisimulation}\\

  \noindent\begin{lemma}\label{lem:B:relations}
   Let  $\Hb_0, \Hb_1,\ldots$ be the ones in the definition of section \ref{forgetting}.1.
   Then,  for each $i\ge 0$,
   \begin{enumerate}[(i)]
      \item $\Hb_{i+1}\subseteq \Hb_i$;
      \item there is a (smallest) $k\ge 0$ such that $\Hb_{k+1}=\Hb_k$;
      \item $\Hb_i$ is reflexive, symmetric and transitive.
   \end{enumerate}
 \end{lemma}
 \begin{proof}
   (i)
   Base: it is clear for $i=0$ by the above definition.

   Step: suppose it holds for $i=n$, i.e., $\Hb_{n+1}\subseteq\Hb_n$. \\
   $(s,s')\in\Hb_{n+2}$\\
   $\Rto$ (a) $(s,s')\in  \Hb_0$,
     (b) for every $(s,s_1)\in R$, there is $(s',s_1')\in R'$
      such that $(s_1,s_1')\in \Hb_{n+1}$, and
     (c)  for every $(s',s_1')\in R'$, there is $(s,s_1)\in R$
     such that $(s_1,s_1')\in \Hb_{n+1}$\\
   $\Rto$ (a) $(s,s')\in  \Hb_0$,
   (b) for every $(s,s_1)\in R$, there is $(s',s_1')\in R'$
      such that $(s_1,s_1')\in \Hb_{n}$ by inductive assumption, and
   (c)  for every $(s',s_1')\in R'$, there is $(s,s_1)\in R$
     such that $(s_1,s_1')\in \Hb_{n}$ by inductive assumption\\
   $\Rto$ $(s,s')\in \Hb_{n+1}$.

   (ii) and (iii) are evident from (i) and the definition of $\Hb_i$.
 \end{proof}

 \noindent\textbf{Lemma}~\ref{lem:equive}  The relation $\lrto_V$ is an equivalence relation.
 \begin{proof}
 It is clear from Lemma~\ref{lem:B:relations} (ii) such that there is a $k \geq $ 0 where $\Hb_k = \Hb_{k+1}$ which is  $\lrto_V$, and it is reflexive, symmetric and transitive by (iii).
 \end{proof}

 \noindent\textbf{Proposition}~\ref{div}
 Let $i\in \{1,2\}$, $V_1,V_2\subseteq\cal A$, $s_i'$s be two states,
   $\pi_i'$s be two paths
 and ${\cal K}_i=({\cal M}_i,s_i)~(i=1,2,3)$ be \MPK-structures
  such that
 ${\cal K}_1\lrto_{V_1}{\cal K}_2$ and ${\cal K}_2\lrto_{V_2}{\cal K}_3$.
  Then:
  \begin{enumerate}[(i)]
   \item $s_1'\lrto_{V_i}s_2'~(i=1,2)$ implies $s_1'\lrto_{V_1\cup V_2}s_2'$;
   \item $\pi_1'\lrto_{V_i}\pi_2'~(i=1,2)$ implies $\pi_1'\lrto_{V_1\cup V_2}\pi_2'$;
   \item for each path $\pi_{s_1}$ of $\Hm_1$ there is a path $\pi_{s_2}$  of $\Hm_2$ such that $\pi_{s_1} \lrto_{V_1} \pi_{s_2}$, and vice versa;
   \item ${\cal K}_1\lrto_{V_1\cup V_2}{\cal K}_3$;
   \item If $V_1 \subseteq V_2$ then ${\cal K}_1 \lrto_{V_2} {\cal K}_2$.
  \end{enumerate}
 \begin{proof}
 In order to distinguish the relations $\Hb_0, \Hb_1, \dots$ for different set $V \subseteq \Ha$, by $\Hb_i^V$ we mean the relation $\Hb_1, \Hb_2, \dots$ for $V \subseteq \Ha$.
 Denote as $\Hb_0, \Hb_1, \dots$ when the underlying set $V$ is clear from the context. Moreover, for the ease of notation, we will refer to $\lrto_V$ by $\Hb$ (i.e., without subindex).

 (i) Base: it is clear for $n = 0$.\\
 Step: For $n > 0$, supposing if $({\cal K}_1, {\cal K}_2) \in \Hb_i^{V_1}$ and $({\cal K}_1, {\cal K}_2) \in \Hb_i^{V_2}$ then $({\cal K}_1, {\cal K}_2) \in \Hb_i^{V_1 \cup V_2}$ for all $0 \leq i \leq n$. We will show that if $({\cal K}_1, {\cal K}_2) \in \Hb_{n+1}^{V_1}$ and $({\cal K}_1, {\cal K}_2) \in \Hb_{n+1}^{V_2}$ then $({\cal K}_1, {\cal K}_2) \in \Hb_{n+1}^{V_1 \cup V_2}$.\\
 (a) It is evident that $L_1(s_1) - (V_1 \cup V_2) = L_2(s_2) - (V_1\cup V_2)$.\\
 (b) We will show that for each $(s_1, s_1^1) \in R_1$ there is a $(s_2, s_2^1) \in R_2$ such that $(s_1^1, s_2^1) \in \Hb_n^{V_1 \cup V_2}$. There is $({\cal K}_1^1, {\cal K}_2^1) \in \Hb_{n-1}^{V_1 \cup V_2}$
 due to $({\cal K}_1, {\cal K}_2) \in \Hb_n^{V_1 \cup V_2}$ by inductive assumption. Then we only need to prove for each $(s_1^1, s_1^2) \in R_1$ there is a $(s_2^1, s_2^2) \in R_2$ such that $({\cal K}_1^2, {\cal K}_2^2) \in \Hb_{n-2}^{V_1 \cup V_2}$ and for each $(s_2^1, s_2^2) \in R_2$ there is a $(s_1^1, s_1^2) \in R_1$ such that $({\cal K}_1^2, {\cal K}_2^2) \in \Hb_{n-2}^{V_1 \cup V_2}$. Therefore, we only need to prove that for each $(s_1^n, s_1^{n+1}) \in R_1$ there is a $(s_2^n, s_2^{n+1}) \in R_2$ such that $({\cal K}_1^{n+1}, {\cal K}_2^{n+1}) \in \Hb_0^{V_1 \cup V_2}$ and for each $(s_2^n, s_2^{n+1}) \in R_2$ there is a $(s_1^n, s_1^{n+1}) \in R_1$ such that $({\cal K}_1^{n+1}, {\cal K}_2^{n+1}) \in \Hb_0^{V_1 \cup V_2}$. It is evident that $L_1(s_1^{n+1}) - (V_1 \cup V_2) = L_1(s_2^{n+1}) - (V_1 \cup V_2)$ due to $({\cal K}_1, {\cal K}_2) \in \Hb_{n+1}^{V_1}$ and $({\cal K}_1, {\cal K}_2) \in \Hb_{n+1}^{V_2}$.
 Where ${\cal K}_i^j = (\Hm_i, s_i^j)$ with $i \in \{1, 2\}$ and $0 < j \leq n+1$.\\
 (c) It is similar with (b).

 (ii) It is clear from (i).

 (iii)
 The following property show our result directly.
 Let $V\subseteq\cal A$
 and ${\cal K}_i=({\cal M}_i,s_i)~(i=1,2)$ be \MPK-structures.
 Then $({\cal K}_1,{\cal K}_2)\in\cal B$ if and only if
   \begin{enumerate}[(a)]
     \item $L_1(s_1)- V = L_2(s_2)- V$,
     \item for every $(s_1,s_1')\in R_1$, there is $(s_2,s_2')\in R_2$
     such that $({\cal K}_1',{\cal K}_2')\in \Hb$, and
     \item for every $(s_2,s_2')\in R_2$, there is $(s_1,s_1')\in R_1$
     such that $({\cal K}_1',{\cal K}_2')\in \Hb$,
   \end{enumerate}
  where ${\cal K}_i'=({\cal M}_i,s_i')$ with $i\in\{1,2\}$.

  We prove it from the following two aspects:

  $(\Rto)$
 (a) It is evident that $L_1(s_1)- V = L_2(s_2)- V$;
 (b) 
 $({\cal K}_1, {\cal K}_2) \in \Hb$ iff $({\cal K}_1, {\cal K}_2) \in \Hb_i$ for all $i \geq 0$, then for each $(s_1, s_1') \in R_1$, there is a $(s_2, s_2')\in R_2$  such that  $({\cal K}_1', {\cal K}_2') \in \Hb_{i-1}$ for all $i > 0$ and then $L_1(s_1')- V = L_2(s_2')- V$. Therefore, $({\cal K}_1', {\cal K}_2') \in \Hb$.
 (c) 
  This is similar with (b).

 $(\Lto)$ Obviously, $L_1(s_1)- V = L_2(s_2)- V$ implies that $(s_1, s_2) \in \Hb_0$;
  (b) implies that for every $(s_1,s_1')\in R_1$, there is $(s_2,s_2')\in R_2$
     such that $({\cal K}_1',{\cal K}_2')\in \Hb_i$ for all $i \geq 0$;
 (c) implies that for every $(s_2,s_2')\in R_2$, there is $(s_1,s_1')\in R_1$
     such that $({\cal K}_1',{\cal K}_2')\in \Hb_i$ for all $i \geq 0$\\
 $\Rto$ $({\cal K}_1, {\cal K}_2) \in \Hb_i$ for all $i \geq 0$\\
 $\Rto$ $({\cal K}_1,{\cal K}_2)\in\cal B$.

 (iv) Let ${\cal M}_i=(S_i,R_i,L_i,s_i)~(i=1,2,3)$, $s_1 \lrto_{V_1} s_2$ via a binary relation $\Hb$, and $s_2 \lrto_{V_2} s_3$ via a binary relation $\Hb''$. Let $\Hb' = \{(w_1, w_3)| (w_1, w_2)\in \Hb$ and $(w_2, w_3)\in \Hb_2\}$. It's evident that $(s_1, s_3) \in \Hb'$. We prove $\Hb'$ is a $V_1 \cup V_2$-bisimulation containing $(s_1, s_3)$ from the (a), (b) and (c) of the previous step (iii) of $X$-bisimulation (where $X$ is a set of atoms). For all $(w_1, w_3) \in \Hb'$:
 \begin{enumerate}[(a)]
   \item there exists $w_2 \in S_2$ such that $(w_1,w_2)\in \Hb$ and $(w_2, w_3)\in \Hb''$, and for all $q \notin V_1$, $q \in L_1(w_1)$ iff $q \in L_2(w_2)$ by $w_1 \lrto_{V_1} w_2$ and for all $q' \notin V_2$, $q'\in L_2(w_2)$ iff $q'\in L_3(w_3)$ by $w_2 \lrto_{V_2} w_3$. Then we have for all $r\notin V_1 \cup V_2$, $r \in L_1(w_1)$ iff $r \in L_3(w_3)$.
   \item if $(w_1, u_1) \in \Hr_1$, then there exists $u_2\in S_2$ such that $(w_2, u_2) \in \Hr_2$ and $(u_1,u_2)\in \Hb$ (due to $(w_1,w_2)\in \Hb$ and $(w_2, w_3) \in \Hb''$ by the definition of $\Hb'$); and then there exists $u_3 \in S_3$ such that $(w_3, u_3) \in \Hr_3$ and $(u_2, u_3) \in \Hb''$, hence $(u_1, u_3) \in \Hb'$ by the definition of $\Hb'$.
   \item if $(w_3, u_3) \in \Hr_3$, then there exists $u_2\in S_2$ such that $(w_2, u_2) \in \Hr_2$ and $(u_2, u_3) \in \Hb_2$; and then there exists $u_1 \in S_1$ such that $(w_1, u_1) \in \Hr_1$ and $(u_1, u_2) \in \Hb$, hence $(u_1, u_3) \in \Hb'$ by the definition of $\Hb'$.
 \end{enumerate}

 (v) Let ${\cal K}_{i, j}=(\Hm_i, s_{i,j})$ and $(s_{i, k}, s_{i, k+1}) \in R_i$ mean that $s_{i, k+1}$ is the $(k+2)$-th node in the path
  $(s_i, s_{i, 1}, s_{i,2}, \dots , s_{i, k+1}, \dots)$ ($i=1,2$).
 We will show that $({\cal K}_1, {\cal K}_2) \in \Hb_n^{V_2}$ for all $n \ge 0$ inductively.

 Base: $L_1(s_1) - V_1 = L_2(s_2) - V_1$\\
 $\Rto$ for all $q \in {\cal A} - V_1$ there is $q \in L_1(s_1)$ iff $q \in L_2(s_2)$\\
 $\Rto$ for all $q \in {\cal A} - V_2$ there is $q \in L_1(s_1)$ iff $q \in L_2(s_2)$ due to $V_1 \subseteq V_2$\\
 $\Rto$ $L_1(s_1) - V_2 = L_2(s_2) - V_2$, i.e.,\ $({\cal K}_1, {\cal K}_2) \in \Hb_0^{V_2}$.

 Step: Supposing that $({\cal K}_1, {\cal K}_2) \in \Hb_i^{V_2}$ for all $0 \leq i \leq k$ ($k > 0)$, we will show $({\cal K}_1, {\cal K}_2) \in \Hb_{k+1}^{V_2}$.
 \begin{enumerate} [(a)]
   \item It is evident that $L_1(s_1) - V_2 = L_2(s_2) - V_2$ by base.
   \item For all $(s_1, s_{1,1}) \in R_1$, we will show that there is a $(s_2, s_{2, 1}) \in R_2$ s.t.\ $({\cal K}_{1,1}, {\cal K}_{2,1})\in \Hb_k^{V_2}$. $({\cal K}_{1,1}, {\cal K}_{2,1})\in \Hb_{k-1}^{V_2}$ by inductive assumption, we need only to prove the following points:\\
       (a) For all $(s_{1, k}, s_{1, k+1}) \in R_1$ there is a $(s_{2, k}, s_{2, k+1})\in R_2$ s.t.\ $({\cal K}_{1,k+1}, {\cal K}_{2,k+1})\in \Hb_0^{V_2}$ due to $({\cal K}_{1,1}, {\cal K}_{2,1})\in \Hb_{k}^{V_1}$. It is easy to see that $L_1(s_{1, k+1}) - V_1 = L_1(s_{2, k+1}) - V_1$, then there is $L_1(s_{1, k+1})- V_2 = L_1(s_{2, k+1}) - V_2$. Therefore, $({\cal K}_{1,k+1}, {\cal K}_{2,k+1})\in \Hb_0^{V_2}$.\\
       (b) For all $(s_{2, k}, s_{2, k+1}) \in R_1$ there is a $(s_{1, k}, s_{1, k+1}) \in R_1$ s.t.\ $({\cal K}_{1,k+1}, {\cal K}_{2,k+1})\in \Hb_0^{V_2}$ due to $({\cal K}_{1,1}, {\cal K}_{2,1})\in \Hb_{k}^{V_1}$. This can be proved as (a).
   \item For all $(s_2, s_{2,1}) \in R_1$, we will show that there is a $(s_1, s_{1, 1}) \in R_2$ s.t.\ $({\cal K}_{1,1}, {\cal K}_{2,1})\in \Hb_k^{V_2}$. This can be proved as (ii).
 \end{enumerate}

 \end{proof}


 \noindent\textbf{Theorem}\ref{thm:V-bisimulation:EQ}
 Let $V\subseteq\cal A$, ${\cal K}_i~(i=1,2)$ be two \MPK-structures such that
   ${\cal K}_1\lrto_V{\cal K}_2$ and $\phi$ a formula with $\IR(\phi,V)$. Then
   ${\cal K}_1\models\phi$ if and only if ${\cal K}_2\models\phi$.
 \begin{proof}
 This theorem can be proved by inducting on the formula $\phi$ and supposing $\Var(\phi) \cap V = \Empty$.
 Let ${\cal K}_1 = (\Hm, s)$ and ${\cal K}_2 = (\Hm', s')$.


 \textbf{Case} $\phi = p$ where $p \in \Ha - V$:\\
 $(\Hm, s) \models \phi$ iff $p\in L(s)$  \hfill  (by the definition of satisfiability) \\
 $\LRto$ $p \in L'(s')$ \hfill ($s \lrto_V s'$)\\
 $\LRto$ $(\Hm', s') \models \phi$

 \textbf{Case} $\phi = \neg \psi$:\\
 $(\Hm, s) \models \phi$ iff $(\Hm, s) \not \models \psi$ \\
 $\LRto$ $(\Hm', s') \not \models \psi$  \hfill   (induction hypothesis)\\
 $\LRto$ $(\Hm', s') \models \phi$

 \textbf{Case} $\phi = \psi_1 \vee \psi_2$:\\
 $(\Hm, s) \models \phi$\\
 $\LRto$ $(\Hm, s) \models \psi_1$ or $(\Hm, s) \models \psi_2$\\
 $\LRto$ $(\Hm', s') \models \psi_1$ or $(\Hm', s') \models \psi_2$   \hfill  (induction hypothesis)\\
 $\LRto$ $(\Hm', s') \models \phi$

 \textbf{Case} $\phi = \EXIST \NEXT \psi$:\\
 $\Hm, s \models \phi$ \\
 $\LRto$ There is a path $\pi = (s, s_1, ...)$ such that $\Hm, s_1 \models \psi$\\
 $\LRto$ There is a path $\pi' = (s', s_1', ...)$ such that $\pi \lrto_V \pi'$ \hfill   ($s \lrto_V s'$, Proposition~\ref{div})\\
 $\LRto$ $s_1 \lrto_V s_1'$  \hfill ($\pi \lrto_V \pi'$)\\
 $\LRto$ $(\Hm', s_1') \models \psi$  \hfill  (induction hypothesis)\\
 $\LRto$ $(\Hm', s') \models \phi$

 \textbf{Case} $\phi = \EXIST \GLOBAL \psi$:\\
 $\Hm, s \models \phi$ \\
 $\LRto$ There is a path $\pi =(s=s_0, s_1, ...)$ such that for each $i \geq 0$ there is $(\Hm, s_i) \models \psi$\\
 $\LRto$ There is a path $\pi' = (s'=s_0', s_1', ...)$ such that $\pi \lrto_V \pi'$   \hfill ($s \lrto_V s'$, Proposition~\ref{div})\\
 $\LRto$ $s_i \lrto_V s_i'$ for each $i \geq 0$ \hfill ($\pi \lrto_V \pi'$)\\
 $\LRto$ $(\Hm', s_i') \models \psi$ for each $i \geq 0$  \hfill  (induction hypothesis)\\
 $\LRto$ $(\Hm', s') \models \phi$

 \textbf{Case} $\phi = \EXIST [\psi_1 \UNTIL \psi_2]$:\\
 $\Hm, s \models \phi$ \\
 $\LRto$ There is a path $\pi= (s=s_0, s_1, ...)$ such that there is $i \geq 0$ such that $(\Hm, s_i) \models \psi_2$, and for all $0 \leq j < i$, $(\Hm, s_j) \models \psi_1$\\
 $\LRto$ There is a path $\pi' = (s=s_0', s_1', ...)$ such that $\pi \lrto_V \pi'$  \hfill  ($s \lrto_V s'$, Proposition~\ref{div})\\
 $\LRto$ $(\Hm', s_i') \models \psi_2$, and for all $0 \leq j < i$ $(\Hm', s_j') \models \psi_1$   \hfill   (induction hypothesis)\\
 $\LRto$ $(\Hm', s') \models \phi$
 \end{proof}

 \noindent\textbf{Proposition}~\ref{B_to_T}  Let $V\subseteq\cal A$ and $({\cal M}_i,s_i)~(i=1,2)$ be two \MPK-structures.
   Then
   \[(s_1,s_2)\in{\cal B}_n\mbox{ iff }
   \Tr_j(s_1)\lrto_V\Tr_j(s_2)\mbox{ for every $0\le j\le n$}.\]
 \begin{proof}
 We will prove this from two aspects:

 $(\Rto)$ If $(s_1, s_2) \in \Hb_n$, then $Tr_j(s_1) \lrto_V Tr_j(s_2)$ for all $0 \leq j \leq n$. $(s, s') \in \Hb_n$ implies both roots of $Tr_n(s_1)$ and $Tr_n(s_2)$ have the same atoms except those atoms in $V$.
 Besides, for any $s_{1,1}$ with $(s_1, s_{1,1}) \in R_1$, there is a $s_{2,1}$ with $(s_2, s_{2,1})\in R_2$ s.t. $(s_{1,1}, s_{2,1}) \in \Hb_{n-1}$ and vice versa.
 Then we have $Tr_1(s_1) \lrto_V Tr_1(s_2)$.
 Therefore,  $Tr_n(s_1) \lrto_V Tr_n(s_2)$ by use such method recursively, and then $Tr_j(s_1) \lrto_V Tr_j(s_2)$ for all $0 \leq j \leq n$.

 $(\Lto)$ If $Tr_j(s_1) \lrto_V Tr_j(s_2)$ for all $0\leq j \leq n$, then $(s_1, s_2) \in \Hb_n$.
 $Tr_0(s_1) \lrto_V Tr_0(s_2)$ implies $L(s_1) - V = L'(s_2) - V$ and then $(s, s') \in \Hb_0$.
 $Tr_1(s_1) \lrto_V Tr_1(s_2)$ implies $L(s_1) - V = L'(s_2)- V$ and for every successors $s$ of the root of one, it is possible to find a successor of the root of the other $s'$ such that
 $(s, s')\in \Hb_0$. Therefore $(s_1, s_2) \in \Hb_1$, and then we will have $(s_1, s_2) \in \Hb_n$ by use such method recursively.
 \end{proof}

 \noindent\textbf{Proposition}~\ref{pro:k}   Let $V\subseteq \Ha$, $\Hm$ be an initial structure and $s,s'\in S$
   such that $s\not\lrto_V s'$.
   There exists a least  $k$ such that
   $\Tr_k(s)$ and $\Tr_k(s')$ are not $V$-bisimilar.
 \begin{proof}
 If $s\not\lrto_V s'$, then there exists a least constant $c$ such that $(s_i, s_j) \notin \Hb_c$, and then there is a least constant $m$ ($m \leq c$) such that $\Tr_m(s_i)$ and $\Tr_m(s_j)$ are not V-bisimilar by Proposition~\ref{B_to_T}. Let $k=m$, the lemma is proved.
 \end{proof}

 \noindent\textbf{Section \ref{forgetting}.2  Characterization of initial \MPK-structure}\\

\noindent \textbf{Lemma}\ref{lem:Vb:TrFormula:Equ} Let $V\subseteq \Ha$, $\Hm$ and $\Hm'$ be two initial structures,
 $s\in S$, $s'\in S'$ and $n\ge 0$. If $\Tr_n(s) \lrto_{\overline V} \Tr_n(s')$, then ${\cal F}_V(\Tr_n(s)) \equiv {\cal F}_V(\Tr_n(s'))$.\\
 \begin{proof}
 This result can be proved by inducting on $n$.

 \textbf{Base.} It is evident that for any $s_n\in S$ and $s_n' \in S'$, if $\Tr_0(s_n) \lrto_{\overline V} \Tr_0(s_n')$ then ${\cal F}_V(\Tr_0(s_n)) \equiv {\cal F}_V(\Tr_0(s_n'))$ due to $L(s_n) - \overline V = L'(s_n') - \overline V$ by the definition of the $V$-bisimulation.

 \textbf{Step.} Supposing that for $k=m$ $(0< m \leq n)$ there is if $\Tr_{n-k}(s_k) \lrto_{\overline V} \Tr_{n-k}(s_k')$ then ${\cal F}_V(\Tr_{n-k}(s_k)) \equiv {\cal F}_V(\Tr_{n-k}(s_k'))$, then we will show if $\Tr_{n-k+1}(s_{k-1}) \lrto_{\overline V} \Tr_{n-k+1}(s_{k-1}')$ then ${\cal F}_V(\Tr_{n-k+1}(s_{k-1})) \equiv {\cal F}_V(\Tr_{n-k+1}(s_{k-1}'))$. Obviously that:\\
  ${\cal F}_V(\Tr_{n-k+1}(s_{k-1})) =$
  $\left(\bigwedge_{(s_{k-1},s_k)\in R}
     \EXIST \NEXT {\cal F}_V(\Tr_{n-k}(s_k))\right)
     \wedge \ALL \NEXT\left(\bigvee_{(s_{k-1},s_k)\in R}
     {\cal F}_V(\Tr_{n-k}(s_k) )\right)
     \wedge {\cal F}_V(\Tr_0(s_{k-1}))$\\
  ${\cal F}_V(\Tr_{n-k+1}(s_{k-1}')) =$
  $\left(\bigwedge_{(s_{k-1}',s_k')\in R}
     \EXIST \NEXT {\cal F}_V(\Tr_{n-k}(s_k'))\right)
     \wedge \ALL \NEXT\left(\bigvee_{(s_{k-1}',s_k')\in R}
     {\cal F}_V(\Tr_{n-k}(s_k') )\right)
     \wedge {\cal F}_V(\Tr_0(s_{k-1}'))$ by the definition of characterizing formula of the computation tree.
  Then we have for any $(s_{k-1}, s_k) \in R$ there is $(s_{k-1}', s_k') \in R'$ such that $\Tr_{n-k}(s_k) \lrto_{\overline V} \Tr_{n-k}(s_k')$ by $\Tr_{n-k+1}(s_{k-1}) \lrto_{\overline V} \Tr_{n-k+1}(s_{k-1}')$. Besides, for any $(s_{k-1}', s_k') \in R'$ there is $(s_{k-1}, s_k) \in R$ such that $\Tr_{n-k}(s_k) \lrto_{\overline V} \Tr_{n-k}(s_k')$ by $\Tr_{n-k+1}(s_{k-1}) \lrto_{\overline V} \Tr_{n-k+1}(s_{k-1}')$.
  Therefore, we have ${\cal F}_V(\Tr_{n-k+1}(s_{k-1})) \equiv {\cal F}_V(\Tr_{n-k+1}(s_{k-1}'))$ by induction hypothesis.
 \end{proof}


 \noindent\textbf{Theorem}~\ref{CF}
 Let $V\subseteq \Ha$, $\Hm=(S,R,L,s_0)$ and $\Hm'=(S',R', L',s_0')$ be two initial structures. Then,
 \begin{enumerate}[(i)]
 \item  $(\Hm',s_0') \models {\cal F}_V({\cal M},s_0)
 \text{ iff } ({\cal M},s_0) \lrto_{\overline V} ({\cal M}',s_0')$;

 \item  $s_0 \lrto_{\overline V} s_0'$ implies  ${\cal F}_V(\Hm, s_0) \equiv {\cal F}_V(\Hm', s_0')$.

 \end{enumerate}

 In order to prove Theorem~\ref{CF}, we prove the following two lemmas at first.

 \begin{lemma}\label{Bn:to:Tn}
 Let $V\subseteq \Ha$, $\Hm=(S, R, L,s_0)$ and $\Hm'=(S', R', L',s_0')$ be two initial structures,
 $s\in S$, $s'\in S'$ and $n\ge 0$.
 \begin{enumerate}[(i)]
   \item $({\cal M},s)\models{\cal F}_V(\Tr_n(s))$.
   \item If $({\cal M},s)\models{\cal F}_V(\Tr_n(s'))$ then
   $\Tr_n(s) \lrto_{\overline V} \Tr_n(s')$.
 \end{enumerate}
 \end{lemma}
 \begin{proof}
 (i) It is evident from the definition of ${\cal F}_V(\Tr_n(s))$.
 Base. It is evident that $({\cal M},s)\models {\cal F}_V(\Tr_0(s))$.\\
 Step. For $k \geq 0$, supposing the result talked in (i) is correct in $k - 1$, we will show that $({\cal M},s)\models {\cal F}_V(\Tr_{k+1}(s))$, i.e.,:
 \begin{equation*}
 \resizebox{.91\linewidth}{!}{$
     \displaystyle
  ({\cal M},s)\models \left(\bigwedge_{(s,s')\in R}
     \EXIST \NEXT T(s')\right)
     \wedge \ALL \NEXT\left(\bigvee_{(s,s')\in R}
     T(s')\right)
     \wedge {\cal F}_V(\Tr_0(s)).
  $}
 \end{equation*}
 Where $T(s') ={\cal F}_V(\Tr_k(s'))$. It is evident that $({\cal M},s)\models {\cal F}_V(\Tr_0(s))$ by Base. It is evident that for any $(s,s') \in R$, there is $({\cal M}, s') \models {\cal F}_V(\Tr_k(s'))$ by inductive assumption. Then we have $({\cal M},s)\models \EXIST \NEXT {\cal F}_V(\Tr_k(s')$, and then $({\cal M},s)\models \left(\bigwedge_{(s,s')\in R}
     \EXIST \NEXT {\cal F}_V(\Tr_k(s'))\right)$. Similarly, we have that for any $(s,s') \in R$, there is $({\cal M}, s') \models \bigvee_{(s,s'')\in R}
     {\cal F}_V(\Tr_k(s'') )$. Therefore, $({\cal M},s)\models \ALL \NEXT\left(\bigvee_{(s,s'')\in R}
     {\cal F}_V(\Tr_k(s'') )\right)$.

 (ii)  \textbf{Base}. If $n=0$, then $(\Hm, s)  \models {\cal F}_V(\Tr_0(s'))$ implies $L(s) - \overline V = L'(s') - \overline V$. Hence, $\Tr_0(s) \lrto_{\overline V} \Tr_0(s')$.\\
     \textbf{Step}. Supposing $n>0$ and the result talked in (ii) is correct in $n-1$.\\
   (a) It is easy to see that $L(s) - \overline V = L'(s') - \overline V$.\\
   (b) We will show that for each $(s, s_1) \in R$, there is a $(s', s_1') \in R'$ such that $\Tr_{n-1}(s_1) \lrto_{\overline V} \Tr_{n-1}(s_1')$.
       Since $(\Hm, s) \models {\cal F}_V(\Tr_n(s'))$, then $(\Hm, s) \models \ALL \NEXT\left(\bigvee_{(s',s_1')\in R}{\cal F}_V(\Tr_{n-1}(s_1') )\right)$.
       Therefore, for each $(s, s_1) \in R$ there is a $(s', s_1') \in R'$ such that $(\Hm, s_1) \models {\cal F}_V(\Tr_{n-1}(s_1') )$. Hence, $\Tr_{n-1}(s_1) \lrto_{\overline V} \Tr_{n-1}(s_1')$ by inductive hypothesis.\\
   (c) We will show that for each $(s',s_1')\in R'$ there is a $(s,s_1)\in R$ such that $\Tr_{n-1}(s_1') \lrto_{\overline V} \Tr_{n-1}(s_1)$.
       Since $(\Hm, s) \models {\cal F}_V(\Tr_n(s'))$, then $(\Hm, s) \models  \bigwedge_{(s',s_1')\in R'} \EXIST \NEXT {\cal F}_V(\Tr_{n-1}(s_1'))$.
       Therefore, for each $(s',s_1')\in R'$ there is a $(s,s_1)\in R$ such that $(\Hm, s_1) \models {\cal F}_V(\Tr_{n-1}(s_1')$.
       Hence, $\Tr_{n-1}(s_1) \lrto_{\overline V} \Tr_{n-1}(s_1')$ by inductive hypothesis.
 \end{proof}

 A consequence of the previous lemma is:

 \begin{lemma}\label{div_s}
 Let $V\subseteq \Ha$, $\Hm=(S,R,L,s_0)$ an initial structure, $k={ch({\cal M},V)}$ and $s\in S$.
 \begin{enumerate}[(i)]
   \item $(\Hm, s)\models {\cal F}_V(\Tr_k(s))$, and
   \item for each $s'\in S$, $({\cal M},s) \lrto_{\overline V} ({\cal M},s')$
   if and only if $({\cal M},s')\models{\cal F}_V(\Tr_k(s))$.
 \end{enumerate}
 \end{lemma}
 \begin{proof}
 (i) It is evident from the (i) of Lemma~\ref{Bn:to:Tn}.

 (ii) Let $\phi = {\cal F}_V(\Tr_k(s))$, where $k$ is the V-characteristic number of $\Hm$. $(\Hm, s) \models \phi$ by the definition of ${\cal F}$, and then for all $s' \in S$, if $s \lrto_{\overline V} s'$ there is $(\Hm, s') \models \phi$ by Theorem~\ref{thm:V-bisimulation:EQ} due to $\IR(\phi, \Ha - V)$. Supposing $(\Hm, s')\models \phi$, if $s \nleftrightarrow_{\overline V} s'$, then $\Tr_k(s) \not \lrto_{\overline V} \Tr_k(s')$, and then $(\Hm, s')\not \models \phi$ by Lemma~\ref{Bn:to:Tn}, a contradiction.
 \end{proof}

 Now we are in the position of proving Theorem~\ref{CF}.\\
 \begin{proof}
 (i) Let ${\cal F}_V(\Hm, s_0)$ be the characterizing formula of $(\Hm, s_0)$ on $V$.
 It is evident that $\IR({\cal F}_V(\Hm, s_0), \overline V)$. We will show that $(\Hm, s_0) \models {\cal F}_V(\Hm, s_0)$ at first.

 It is evident that $(\Hm, s_0) \models {\cal F}_V(\Tr_c(s_0))$ by Lemma~\ref{Bn:to:Tn}.
 We must show that $(\Hm, s_0) \models \bigwedge_{s\in S} G(\Hm, s)$.
 Let ${\cal X} = {\cal F}_V(\Tr_c(s)) \rto \left(\bigwedge_{(s,s_1) \in R} \EXIST \NEXT {\cal F}_V(\Tr_c(s_1))\right)$ $\wedge \ALL \NEXT \left(\bigvee_{(s,s_1) \in R} {\cal F}_V(\Tr_c(s_1))\right)$, we will show for all $s\in S$, $(\Hm, s_0) \models G(\Hm, s)$. Where $G(\Hm, s)=\ALL\GLOBAL \cal X$.
 There are two cases we should consider:
 \begin{itemize}
   \item  If $(\Hm, s_0) \not \models {\cal F}_V(\Tr_c(s))$, it is evident that $(\Hm, s_0) \models {\cal X}$;
   \item  If $(\Hm, s_0) \models {\cal F}_V(\Tr_c(s))$:\\
          $(\Hm, s_0) \models {\cal F}_V(\Tr_c(s))$\\
         $\Rto$  $s_0 \lrto_{\overline V} s$ by the definition of characteristic number and Lemma~\ref{div_s}.

         For each $(s, s_1)\in R$ there is:\\
          $(\Hm, s_1) \models {\cal F}_V(\Tr_c(s_1))$  \hfill  ($s_1 \lrto_{\overline V} s_1$)\\
         $\Rto$ $(\Hm, s) \models \bigwedge_{(s,s_1)\in R}\EXIST \NEXT {\cal F}_V(\Tr_c(s_1))$\\
         $\Rto$ $(\Hm, s_0) \models$ $\bigwedge_{(s,s_1)\in R}\EXIST \NEXT {\cal F}_V(\Tr_c(s_1))$    \qquad  (by $\IR(\bigwedge_{(s,s_1)\in R}\EXIST \NEXT {\cal F}_V(\Tr_c(s_1)), \overline V)$, $s_0 \lrto_{\overline V} s$).

          For each $(s, s_1)$ there is:\\
           $\Hm, s_1 \models \bigvee_{(s, s_2)\in R}{\cal F}_V(\Tr_c(s_2))$\\
         $\Rto$ $(\Hm, s) \models \ALL \NEXT \left( \bigvee_{(s, s_2)\in R} {\cal F}_V(\Tr_c(s_2)) \right)$ \\
         $\Rto$ $(\Hm, s_0) \models$  $\ALL \NEXT \left( \bigvee_{(s, s_2)\in R} {\cal F}_V(\Tr_c(s_2)) \right)$   \qquad  (by $\IR(\ALL \NEXT \left( \bigvee_{(s, s_2)\in R} {\cal F}_V(\Tr_c(s_2)) \right), \overline V)$, $s_0 \lrto_{\overline V} s$)\\
         $\Rto$ $(\Hm, s_0) \models {\cal X}$.\\
 \end{itemize}
 For any other states $s'$ which can reach from $s_0$ can be proved similarly, i.e.,, $(\Hm,s')\models \cal X$.
 Therefore, for all $s\in S$, $(\Hm, s_0) \models G(\Hm, s)$, and then $(\Hm, s_0) \models {\cal F}_V(\Hm, s_0)$.

 We will prove this theorem from the following two aspects:

 $(\Lto)$ If $s_0 \lrto_{\overline V} s_0'$, then $(\Hm',s_0') \models {\cal F}_V(M,s_0)$. Since $(\Hm, s_0) \models {\cal F}_V(\Hm, s_0)$ and $\IR({\cal F}_V(\Hm, s_0), \overline V)$, hence
 $(\Hm',s_0') \models {\cal F}_V(M,s_0)$ by Theorem~\ref{thm:V-bisimulation:EQ}.

 $(\Rto)$ If $(\Hm',s_0') \models {\cal F}_V(M,s_0)$, then $s_0 \lrto_{\overline V} s_0'$. We will prove this by showing that for all $n \geq 0$, $Tr_n(s_0) \lrto_{\overline V} Tr_n(s_0')$.

 \textbf{Base}. It is evident that $Tr_0(s_0) \equiv Tr_0(s_0')$.

 \textbf{Step}. Supposing $\Tr_k(s_0) \lrto_{\overline V} \Tr_k(s_0')$ ($k > 0$), we will prove $\Tr_{k+1}(s_0) \lrto_{\overline V} \Tr_{k+1}(s_0')$. We should only show that $\Tr_1(s_k) \lrto_{\overline V} \Tr_1(s_k')$. Where $(s_0, s_1), (s_1, s_2)$, $\dots$, $(s_{k-1}, s_k) \in R$ and $(s_0', s_1'), (s_1', s_2'), \dots, (s_{k-1}', s_k') \in R'$, i.e., $s_{i+1}$ ($s_{i+1}'$) is an immediate successor of $s_i$ ($s_i'$) for all $0 \leq i \leq k-1$.

       (a) It is evident that $L(s_k) - \overline V = L'(s_k') - \overline V$ by inductive assumption.

       Before talking about the other points, note the following fact that:\\
       $(\Hm',s_0') \models {\cal F}_V(\Hm,s_0)$\\
       $\Rto$ For all $s'\in S'$, $(\Hm', s')\models {\cal F}_V(\Tr_c(s)) \rto$ \\ $\left(\bigwedge_{(s,s_1)\in R} \EXIST \NEXT {\cal F}_V(\Tr_c(s_1))\right)\wedge \ALL \NEXT \left( \bigvee_{(s,s_1)\in R} {\cal F}_V(\Tr_c(s_1))\right)$  for any $s\in S$.   \hfill  \textbf{(fact)}\\
       (I) $(\Hm', s_0') \models {\cal F}_V(\Tr_c(s_0)) \rto \left(\bigwedge_{(s_0, s_1) \in R} \EXIST \NEXT {\cal F}_V(\Tr_c(s_1))\right)$ $\wedge$ $\ALL \NEXT \left(\bigvee_{(s_0, s_1) \in R} {\cal F}_V(\Tr_c(s_1)) \right)$     \hfill  \textbf{(fact)}\\
         (II) $(\Hm', s_0') \models {\cal F}_V(\Tr_c(s_0)))$  \hfill  (known)\\
         (III) $(\Hm', s_0') \models \left(\bigwedge_{(s_0, s_1) \in R} \EXIST \NEXT {\cal F}_V(\Tr_c(s_1))\right)$ $\wedge$ $\ALL \NEXT \left(\bigvee_{(s_0, s_1) \in R} {\cal F}_V(\Tr_c(s_1)) \right)$  \hfill  ((I),(II))\\

         (b) We will show that for each $(s_k, s_{k+1}) \in R$ there is a $(s_k', s_{k+1}') \in R'$ such that $L(s_{k+1}) - \overline V = L'(s_{k+1}') - \overline V$.\\
         (1) $(\Hm', s_0') \models \bigwedge_{(s_0, s_1) \in R} \EXIST \NEXT {\cal F}_V(\Tr_c(s_1))$  \hfill  (III)\\
         (2) For all $(s_0, s_1) \in R$, there exists $(s_0', s_1') \in R'$ s.t.\ $(\Hm', s_1') \models {\cal F}_V(\Tr_c(s_1))$  \hfill  (2)\\
         (3) $\Tr_c(s_1) \lrto_{\overline V} \Tr_c(s_1')$  \hfill  ((2), Lemma~\ref{Bn:to:Tn}) \\
         (4) $L(s_1) - \overline V = L'(s_1') - \overline V$  \hfill   ((3), $c \geq 0)$\\
         (5) $(\Hm', s_1') \models {\cal F}_V(\Tr_c(s_1)) \rto \left(\bigwedge_{(s_1,s_2)\in R} \EXIST \NEXT {\cal F}_V(\Tr_c(s_2))\right) \wedge \ALL \NEXT \left(\bigvee_{(s_1,s_2)\in R} {\cal F}_V(\Tr_c(s_2))\right)$     \hfill  \textbf{(fact)}\\
         (6) $(\Hm', s_1') \models \left(\bigwedge_{(s_1,s_2)\in R} \EXIST \NEXT {\cal F}_V(\Tr_c(s_2))\right) \wedge \ALL \NEXT \left(\bigvee_{(s_1,s_2)\in R} {\cal F}_V(\Tr_c(s_2))\right)$ \hfill ((2), (5))\\
         (7) $\dots \dots$ \\
         (8) $(\Hm', s_k') \models \left(\bigwedge_{(s_k,s_{k+1})\in R} \EXIST \NEXT {\cal F}_V(\Tr_c(s_{k+1}))\right) \wedge \ALL \NEXT \left(\bigvee_{(s_k,s_{k+1})\in R} {\cal F}_V(\Tr_c(s_{k+1}))\right)$       \hfill (similar with (6))\\
         (9) For all $(s_k, s_{k+1}) \in R$, there exists $(s_k', s_{k+1}') \in R'$ s.t.\ $(\Hm', s_{k+1}') \models {\cal F}_V(\Tr_c(s_{k+1}))$  \hfill  (8)\\
         (10) $\Tr_c(s_{k+1}) \lrto_{\overline V} \Tr_c(s_{k+1}')$    \hfill ((9), Lemma~\ref{Bn:to:Tn}) \\
         (11) $L(s_{k+1}) - \overline V = L'(s_{k+1}') - \overline V$  \hfill   ((10), $c \geq 0)$\\

         (c) We will show that for each $(s_k', s_{k+1}') \in R'$ there is a $(s_k, s_{k+1})\in R$ such that $L(s_{k+1}) - \overline V = L'(s_{k+1}') - \overline V$.\\
         (1) $(\Hm', s_k') \models \ALL \NEXT \left(\bigvee_{(s_k,s_{k+1})\in R} {\cal F}_V(\Tr_c(s_{k+1}))\right)$  \hfill (by (8) talked above)\\
         (2) For all $(s_k', s_{k+1}') \in R'$, there exists $(s_k, s_{k+1}) \in R$ s.t.\ $(\Hm', s_{k+1}') \models {\cal F}_V(\Tr_c(s_{k+1}'))$  \hfill (1) \\
         (3) $\Tr_c(s_{k+1}) \lrto_{\overline V} \Tr_c(s_{k+1}')$    \hfill ((2), Lemma~\ref{Bn:to:Tn}) \\
         (4) $L(s_{k+1}) - \overline V = L'(s_{k+1}') - \overline V$  \hfill   ((3), $c \geq 0)$\\

 (ii) This is following Lemma~\ref{lem:Vb:TrFormula:Equ} and the definition of the characterizing formula of initial \MPK-structure ${\cal K}$ on $V$.

 \end{proof}

 \noindent\textbf{Lemma}~\ref{lem:models:formula} Let $\varphi$ be a formula. We have
   \begin{equation}
     \varphi\equiv \bigvee_{(\Hm, s_0)\in\Mod(\varphi)}{\cal F}_{\cal A}(\Hm, s_0).
     \end{equation}
 \begin{proof}
 Let $(\Hm', s_0')$ be a model of $\varphi$. Then $(\Hm', s_0') \models \bigvee_{(\Hm, s_0)\in \Mod(\varphi)} {\cal F}_{\Ha}(\Hm, s_0)$ due to $(\Hm', s_0') \models {\cal F}_{\Ha}(\Hm', s_0')$. On the other hand, suppose that $(\Hm', s_0')$ is a model of $\bigvee_{(\Hm, s_0)\in \Mod(\varphi)} {\cal F}_{\Ha}(\Hm, s_0)$. Then there is a $(\Hm, s_0)\in \Mod(\varphi)$ such that $(\Hm', s_0') \models {\cal F}_{\Ha}(\Hm, s_0)$. And then $(\Hm, s_0) \lrto_{\Empty} (\Hm', s_0')$ by Theorem~\ref{CF}. Therefore, $(\Hm, s_0)$ is also a model of $\varphi$ by Theorem~\ref{thm:V-bisimulation:EQ}.
 \end{proof}


 \noindent \textbf{Section \ref{forgetting}.3 Semantic properties of forgetting in CTL}
\\

\noindent\textbf{Theorem}~\ref{thm:PL:CTL}
Let $\varphi$ be a CPL formula and $V\subseteq \Ha$, then
\[
\CTLforget(\varphi, V) \equiv \Forget(\varphi, V).
\]
\begin{proof}
On one hand, for each $(\Hm, s) \in \Mod(\CTLforget(\varphi, V))$ there exists a $(\Hm', s') \in \Mod(\varphi)$ such that $s\lrto_V s'$. Thus, $(s, s') \in \Hb_0^V$. Hence, $(\Hm, s)$ is a model of $\Forget(\varphi, V)$.

On the other hand, for each $(\Hm, s) \in \Mod(\Forget(\varphi, V))$ with $\Hm = (S, R, L, s)$ there exists a $(\Hm', s') \in \Mod(\varphi)$ such that $(s, s') \in \Hb_0^V$. Construct an initial K-structure $(\Hm_1, s_1)$ such that $\Hm_1=(S_1, R_1, L_1, s_1)$ with $S_1= (S - \{s\}) \cup \{s_1\}$, $R_1$ is the same as $R$ except replace $s$ with $s_1$, and $L_1$ is the same as $L$ except $L_1(s_1) = L'(s')$, where $L'$ is the label function of $M'$. It is clear that $(\Hm_1, s_1)$ is a model of $\varphi$ and $s_1 \lrto_V s$. Hence, $(\Hm, s)$ is a model of $\CTLforget(\varphi, V)$.
\end{proof}

\noindent \textbf{Theorem}~\ref{thm:close}
 \textbf{(Representation theorem)}
Let $\varphi$ and $\varphi'$ be \CTL\ formulas and $V \subseteq \Ha$.
The following statements are equivalent:
\begin{enumerate}[(i)]
  \item $\varphi' \equiv \CTLforget(\varphi, V)$,
  \item $\varphi'\equiv \{\phi \mid\varphi \models \phi \text{ and } \IR(\phi, V)\}$,
  \item Postulates (\W), (\PP), (\NgP) and (\textbf{IR}) hold if $\varphi,   \varphi'$ and $V$ are as in (i) and (ii).
\end{enumerate}
 \begin{proof}
 $(i) \LRto (ii)$. To prove this, we will show that:
 \begin{align*}
  & \Mod(\CTLforget(\varphi, V)) = \Mod(\{\phi | \varphi \models \phi, \IR(\phi, V)\})\\
  & = \Mod(\bigvee_{\Hm, s_0\in \Mod(\varphi)} {\cal F}_{\Ha- V}(\Hm, s_0)).
 \end{align*}
 Firstly, suppose that $(\Hm', s_0')$ is a model of $\CTLforget(\varphi, V)$. Then there exists an initial \MPK-structure $(\Hm, s_0)$ such that $(\Hm, s_0)$ is a model of $\varphi$ and $(\Hm, s_0) \lrto_V (\Hm', s_0')$. By Theorem~\ref{thm:V-bisimulation:EQ}, we have $(\Hm', s_0') \models \phi$ for all $\phi$ such that $\varphi\models \phi$ and $\IR(\phi, V)$. Thus, $(\Hm', s_0')$ is a model of $\{\phi | \varphi \models \phi, \IR(\phi, V)\}$.

 Secondly, suppose that $(\Hm', s_0')$ is a models of $\{\phi | \varphi \models \phi, \IR(\phi, V)\}$. Thus, $(\Hm', s_0')$ $\models$ $\bigvee_{(\Hm, s_0)\in \Mod(\varphi)} {\cal F}_{\Ha- V}(\Hm, s_0)$ due to $\bigvee_{(\Hm, s_0)\in \Mod(\varphi)} {\cal F}_{\Ha- V}(\Hm, s_0)$ is irrelevant to $V$ and $\varphi \models$ $\bigvee_{(\Hm, s_0)\in \Mod(\varphi)} {\cal F}_{\Ha- V}(\Hm, s_0)$ by Lemma~\ref{lem:models:formula}.

 Finally, suppose that $(\Hm', s_0')$ is a model of $\bigvee_{\Hm, s_0\in \Mod(\varphi)} {\cal F}_{\Ha- V}(\Hm, s_0)$. Then there exists $(\Hm, s_0) \in \Mod(\varphi)$ such that $(\Hm', s_0') \models {\cal F}_{\Ha- V}(\Hm, s_0)$. Hence, $(\Hm, s_0)$ $\lrto_V$ $(\Hm', s_0')$ by Theorem~\ref{CF}. Thus $(\Hm', s_0')$ is also a model of $\CTLforget(\varphi,V)$.

 $(ii)\Rto (iii)$. For convenience, let $A=\{\phi | \varphi \models \phi \text{ and } \IR(\phi, V)\}$. First, it is easy to see that $\IR(A, V)$ since for any $\phi' \in A$ there is $\IR(\phi', V)$. Therefore, we have $\IR(\varphi',V)$. Second, $\varphi \models \phi'$ for any $\phi'\in A$, hence $\varphi \models \varphi'$. The $(\NgP)$ and $(\PP)$ are obvious from $A$.

 $(iii)\Rto (ii)$. Suppose that all postulates hold. By Positive Persistence, we have $\varphi' \models \{\phi | \varphi \models \phi, \IR(\phi, V)\}$.
 The  $\{\phi \mid \varphi \models \phi, \IR(\phi, V)\} \models \varphi'$ can be obtained from (\W) and (\textbf{IR}).
 Thus, $\varphi'$ is equivalent to $\{\phi | \varphi \models \phi, \IR(\phi, V)\}$.
 \end{proof}

  \noindent \textbf{Lemma}~\ref{lem:KF:eq} Let $\varphi$ and $\alpha$ be two \CTL\ formulae and $q\in
 		\overline{\Var(\varphi) \cup \Var(\alpha)}$. Then
 	$\CTLforget(\varphi \wedge (q\lrto\alpha), q)\equiv \varphi$.\\
     \begin{proof}
 	Let $\varphi' =\varphi \wedge (q\lrto\alpha)$. For any model $({\cal M},s)$ of $\CTLforget(\varphi', q)$ there is an initial \MPK-structure $({\cal M}',s')$ s.t.\ $({\cal M},s)\lrto_{\{q\}}({\cal M}',s')$ and $({\cal M}',s') \models \varphi'$. It's evident that $({\cal M}',s') \models \varphi$, and then $({\cal M},s) \models \varphi$ since $\IR(\varphi,\{q\})$ and $({\cal M},s)\lrto_{\{q\}}({\cal M}',s')$
 	by Theorem~\ref{thm:V-bisimulation:EQ}.

 	Let $(\Hm,s) \in \Mod(\varphi)$ with ${\cal M}=(S, R, L,s)$. We construct $(\Hm', s)$ with $\Hm' = (S, R, L',s)$ as follows:
     \begin{align*}
       & L':S \rto \Ha\ and\ \forall s^*\in S, L'(s^*) = L(s^*)\ if\ (\Hm, s^*) \not \models \alpha,\\
       & else\ L'(s^*) = L(s^*)\cup\{q\}, \\
       & L'(s) = L(s) \cup\{q\}\ if\ (\Hm, s) \models \alpha,\ and\ L'(s) = L(s)\\
       & otherwise.
     \end{align*}
 	It is clear that $({\cal M}',s) \models \varphi$, $({\cal M}',s) \models q\lrto \alpha$ and
 	$({\cal M}', s) \lrto_{\{q\}} ({\cal M}, s)$. Therefore $({\cal M}', s) \models \varphi \wedge (q\lrto\alpha)$, and then $({\cal M}, s) \models \CTLforget (\varphi \wedge (q\lrto\alpha), q)$ by
 	$({\cal M}', s) \lrto_{\{q\}} ({\cal M}, s)$.
 \end{proof}

 \noindent\textbf{Proposition}~\ref{disTF} \textbf{(Modularity)} Given a formula $\varphi \in \CTL$, $V$ a set of atoms and $p$ an atom such that $p \notin V$. Then,
 \[
 \CTLforget(\varphi, \{p\} \cup V) \equiv \CTLforget(\CTLforget(\varphi, p), V).
 \]
 \begin{proof}
 Let $(\Hm_1, s_1) $ with ${\cal M}_1=(S_1, R_1, L_1,s_1)$ be a model of $\CTLforget(\varphi, \{p\} \cup V)$. By the definition, there exists a model $(\Hm,s)$ with ${\cal M} = (S, R,L,s)$ of $\varphi$, such that $(\Hm_1, s_1)$ $\lrto_{\{p\} \cup V}$ $(\Hm, s)$. We construct an initial \MPK-structure $(\Hm_2, s_2)$ with ${\cal M}_2 = (S_2, R_2, L_2,s_2)$ as follows:
 \begin{enumerate}[(1)]
   \item for $s_2$: let $s_2$ be the state such that:
   \begin{itemize}
     \item $p \in L_2(s_2)$ iff $p \in L_1(s_1)$,
     \item for all $q \in V$, $q \in L_2(s_2)$ iff $q\in L(s)$,
     \item for all other atoms $q'$, $q' \in L_2(s_2)$ iff $q' \in L_1(s_1)$ iff $q'\in L(s)$.
   \end{itemize}
   \item for another:
   \begin{enumerate}[(i)]
     \item for all pairs  $w \in S$ and $w_1 \in S_1$ such that $w \lrto_{\{p\} \cup V} w_1$, let $w_2 \in S_2$ and
         \begin{itemize}
           \item $p \in L_2(w_2)$ iff $p \in L_1(w_1)$,
           \item for all $q \in V$, $q \in L_2(w_2)$ iff $q\in L(w)$,
           \item for all other atoms $q'$, $q' \in L_2(w_2)$ iff $q' \in L_1(w_1)$ iff $q'\in L(w)$.
         \end{itemize}
     \item if $(w_1', w_1)\in R_1$, $w_2$ is constructed based on $w_1$ and $w_2'\in S_2$ is constructed based on $w_1'$, then $(w_2', w_2)\in R_2$.
   \end{enumerate}
   \item delete duplicated states in $S_2$ and pairs in $R_2$.
 \end{enumerate}
 Then we have $(\Hm, s) \lrto_{\{p\}} (\Hm_2, s_2)$ and $(\Hm_2, s_2) \lrto_V (\Hm_1, s_1)$. Thus, $(\Hm_2, s_2) \models \CTLforget(\varphi, p)$. And therefore $(\Hm_1, s_1) \models \CTLforget(\CTLforget(\varphi, p), V)$.

 On the other hand, suppose that $(\Hm_1, s_1)$ is a model of $\CTLforget(\CTLforget(\varphi, p), V)$, then there exists an initial \MPK-structure $(\Hm_2, s_2)$ such that $(\Hm_2, s_2) \models \CTLforget(\varphi, p)$ and $(\Hm_2, s_2) \lrto_V (\Hm_1, s_1)$, and there exists $(\Hm, s)$ such that $(\Hm, s) \models \varphi$ and $(\Hm, s) \lrto_{\{p\}} (\Hm_2, s_2)$. Therefore, $(\Hm, s) \lrto_{\{p\} \cup V} (\Hm_1, s_1)$ by Proposition~\ref{div}, and consequently, $(\Hm_1, s_1) \models \CTLforget(\varphi, \{p\} \cup V)$.
 \end{proof}

 \noindent\textbf{Proposition}~\ref{pro:ctl:forget:1}
 Let $\varphi$, $\varphi_i$, $\psi_i$ ($i=1,2$) be formulas in \CTL\ and $V\subseteq \Ha$. We have
 \begin{enumerate}[(i)]
   \item $\CTLforget(\varphi, V)$ is satisfiable iff $\varphi$ is;
   \item If $\varphi_1 \equiv \varphi_2$, then $\CTLforget(\varphi_1, V) \equiv \CTLforget(\varphi_2, V)$;
   \item If $\varphi_1 \models \varphi_2$, then $\CTLforget(\varphi_1, V) \models \CTLforget(\varphi_2, V)$;
   \item $\CTLforget(\psi_1 \vee \psi_2, V) \equiv \CTLforget(\psi_1, V) \vee \CTLforget(\psi_2, V)$;
   \item $\CTLforget(\psi_1 \wedge \psi_2, V) \models \CTLforget(\psi_1, V) \wedge \CTLforget(\psi_2, V)$;
 \end{enumerate}
 \begin{proof}
 (i) ($\Rto$) Supposing $(\Hm, s)$ is a model of $\CTLforget(\varphi, V)$, then there is a model $(\Hm',s')$ of $\varphi$ s.t. $(\Hm,s) \lrto_V (\Hm',s')$ by the definition of $\CTLforget$.

 ($\Lto$) Supposing $(\Hm, s)$ is a model of $\varphi$, then there is an initial \MPK-structure $(\Hm',s')$ s.t. $(\Hm,s) \lrto_V (\Hm',s')$, and then $(\Hm',s') \models \CTLforget(\varphi, V)$ by the definition of $\CTLforget$.

 The (ii) and (iii) can be proved similarly.

 (iv) ($\Rto$) For all$(\Hm,s)\in \Mod(\CTLforget(\psi_1 \vee \psi_2, V))$, there exists $(\Hm',s')$ $\in$  $\Mod(\psi_1\vee \psi_2)$ s.t. $(\Hm,s) \lrto_V (\Hm',s')$ and $(\Hm',s') \models \psi_1$ or $(\Hm',s') \models \psi_2$ \\
 $\Rto$ there exists $(\Hm_1,s_1) \in \Mod(\CTLforget(\psi_1, V))$ s.t. $(\Hm',s') \lrto_V (\Hm_1,s_1)$ or there exists $(\Hm_2,s_2) \in \Mod(\CTLforget(\psi_2, V))$ s.t. $(\Hm',s') \lrto_V (\Hm_2,s_2)$ \\
 $\Rto$ $(\Hm,s) \models \CTLforget(\psi_1, V) \vee \CTLforget(\psi_2, V)$ by Theorem~\ref{thm:V-bisimulation:EQ}.

 ($\Lto$) for all $(\Hm,s) \in \Mod(\CTLforget(\psi_1, V) \vee \CTLforget(\psi_2, V))$\\
 $\Rto$ $(\Hm,s) \models \CTLforget(\psi_1,V)$ or $(\Hm,s) \models \CTLforget(\psi_2,V)$\\
 $\Rto$ there is an initial \MPK-structure $(\Hm_1,s_1)$ s.t. $(\Hm,s) \lrto_V (\Hm_1,s_1)$ and $(\Hm_1,s_1) \models \psi_1$ or  $(\Hm_1,s_1) \models \psi_2$\\
 $\Rto$ $(\Hm_1,s_1) \models \psi_1 \vee \psi_2$\\
 $\Rto$ there is an initial \MPK-structure $(\Hm_2,s_2)$ s.t. $(\Hm_1,s_1) \lrto_V (\Hm_2,s_2)$ and $(\Hm_2,s_2) \models \CTLforget(\psi_1 \vee \psi_2, V)$\\
 $\Rto$ $(\Hm,s) \lrto_V (\Hm_2,s_2)$ and $(\Hm,s) \models \CTLforget(\psi_1 \vee \psi_2, V)$.

 The (v) can be proved as (iv).
 \end{proof}

 \textbf{Proposition}~\ref{pro:ctl:forget:2} \textbf{(Homogeneity)}
 Let $V\subseteq\cal A$ and $\phi \in \CTL$,
   \begin{enumerate}[(i)]
     \item $\CTLforget(\ALL\NEXT\phi,V)\equiv \ALL\NEXT \CTLforget(\phi,V)$.
     \item $\CTLforget(\EXIST\NEXT\phi,V)\equiv\EXIST\NEXT \CTLforget(\phi,V)$.
     \item $\CTLforget(\ALL \FUTURE\phi,V)\equiv \ALL \FUTURE \CTLforget(\phi,V)$.
     \item $\CTLforget(\EXIST\FUTURE\phi,V)\equiv\EXIST\FUTURE \CTLforget(\phi,V)$.
   \end{enumerate}
 \begin{proof}
 Let $\Hm=(S, R, L,s_0)$ with initial state $s_0$ and $\Hm'=(S', R', L',s_0')$ with initial state $s_0'$, then we call $\Hm', s_0'$ be a sub-structure of $\Hm,s_0$ if:
 \begin{itemize}
   \item $S' \subseteq S$ and $S'=\{s' | s'$ is reachable from $s_0'\}$,
   \item $R' =\{(s_1, s_2)| s_1, s_2 \in S'$ and $(s_1, s_2) \in R\}$,
   \item $L': S' \rto 2^\Ha$ and for all $s_1 \in S'$ there is $L'(s_1) = L(s_1)$, and
   \item $s_0'$ is $s_0$ or a state reachable from $s_0$.
 \end{itemize}

 (i) In order to prove $\CTLforget(\ALL \NEXT \phi, V) \equiv \ALL \NEXT(\CTLforget(\phi, V))$, we only need to prove $\Mod(\CTLforget(\ALL \NEXT \phi, V)) = \Mod( \ALL\NEXT\CTLforget(\phi, V))$:

 $(\Rto)$ For all $(\Hm', s') \in \Mod(\CTLforget(\ALL \NEXT \phi, V))$ there exists an initial \MPK-structure $(\Hm, s)$ s.t. $(\Hm, s)\models \ALL \NEXT \phi$ and $(\Hm, s) \lrto_V (\Hm',s')$\\
 $\Rto$ for any sub-structure $(\Hm_1, s_1)$ of $(\Hm, s)$ there is $(\Hm_1, s_1) \models \phi$, where $s_1$ is a directed successor of $s$ \\
 $\Rto$ there is an initial \MPK-structure $(\Hm_2, s_2)$ s.t. $(\Hm_2, s_2) \models \CTLforget(\phi,V)$ and $(\Hm_2, s_2) \lrto_V (\Hm_1,s_1)$\\
 $\Rto$ it is easy to construct an initial \MPK-structure $(\Hm_3, s_3)$ by $(\Hm_2, s_2)$ s.t. $(\Hm_2, s_2)$ is a sub-structure of $(\Hm_3, s_3)$ with $s_2$ is a direct successor of $s_3$ and $(\Hm_3, s_3) \lrto_V (\Hm,s)$\\
 $\Rto$ $(\Hm_3, s_3) \models \ALL \NEXT (\CTLforget(\phi,V))$ and $(\Hm_3, s_3) \lrto_V (\Hm',s')$\\
 $\Rto$ $(\Hm', s') \models \ALL \NEXT (\CTLforget(\phi,V))$.

 $(\Lto)$ For all $(\Hm_3, s_3) \in \Mod(\ALL \NEXT (\CTLforget(\phi,V)))$, then for any sub-structure $(\Hm_2, s_2)$ with $s_2$ is a directed successor of $s_3$ there is $(\Hm_2, s_2) \models \CTLforget(\phi,V)$\\
 $\Rto$ for any $(\Hm_2, s_2)$ there is an initial \MPK-structure $(\Hm_1, s_1)$ s.t. $(\Hm_1, s_1) \models \phi$ and $(\Hm_1, s_1) \lrto_V (\Hm_2, s_2)$\\
 $\Rto$ it is easy to construct an initial \MPK-structure $(\Hm,s)$ by $(\Hm_1, s_1)$ s.t. $(\Hm_1, s_1)$ is a sub-structure of $(\Hm, s)$ with $s_1$ is a direct successor of $s$ and $(\Hm, s)\lrto_V (\Hm_3, s_3)$\\
 $\Rto$ $(\Hm, s) \models \ALL \NEXT \phi$ and then $(\Hm_3, s_3) \models \CTLforget(\ALL \NEXT \phi, V)$.

 (ii) In order to prove $\CTLforget(\EXIST \NEXT \phi, V) \equiv \EXIST\NEXT\CTLforget(\phi, V)$, we only need to prove $\Mod$ $(\CTLforget(\EXIST \NEXT \phi$, $V)) = \Mod( \EXIST\NEXT\CTLforget(\phi, V))$:

 $(\Rto)$ For all $(\Hm', s') \in \Mod(\CTLforget(\EXIST \NEXT \phi, V))$ there exists an initial \MPK-structure $(\Hm, s)$ s.t. $(\Hm, s) \models \EXIST \NEXT \phi$ and $(\Hm, s) \lrto_V (\Hm',s')$\\
 $\Rto$ there is a sub-structure $(\Hm_1, s_1)$ of $(\Hm, s)$ s.t. $(\Hm_1, s_1) \models \phi$, where $s_1$ is a directed successor of $s$\\
 $\Rto$ there is an initial \MPK-structure $(\Hm_2, s_2)$ s.t. $(\Hm_2, s_2) \models \CTLforget(\phi,V)$ and $(\Hm_2, s_2) \lrto_V (\Hm_1,s_1)$\\
 $\Rto$ it is easy to construct an initial \MPK-structure $(\Hm_3, s_3)$ by $(\Hm_2, s_2)$ s.t. $(\Hm_2, s_2)$ is a sub-structure of $(\Hm_3, s_3)$ that $s_2$ is a direct successor of $s_3$ and $(\Hm_3, s_3) \lrto_V (\Hm,s)$\\
 $\Rto$ $(\Hm_3, s_3) \models \EXIST \NEXT (\CTLforget(\phi,V))$\\
 $\Rto$ $(\Hm', s') \models \EXIST \NEXT (\CTLforget(\phi,V))$.

 $(\Lto)$ For all $(\Hm_3, s_3) \in \Mod(\EXIST \NEXT (\CTLforget(\phi,V)))$, there exists a sub-structure $(\Hm_2, s_2)$ of $(\Hm_3, s_3)$ s.t. $(\Hm_2, s_2) \models \CTLforget(\phi,V)$\\
 $\Rto$ there is an initial \MPK-structure $(\Hm_1, s_1)$ s.t. $(\Hm_1, s_1) \models \phi$ and $(\Hm_1, s_1) \lrto_V (\Hm_2, s_2)$\\
 $\Rto$ it is easy to construct an initial \MPK-structure $(\Hm,s)$ by $(\Hm_1, s_1)$ s.t. $(\Hm_1, s_1)$ is a sub-structure of $(\Hm, s)$ that $s_1$ is a direct successor of $s$ and $(\Hm, s)\lrto_V (\Hm_3, s_3)$\\
 $\Rto$ $(\Hm, s) \models \EXIST \NEXT \phi$ and then $(\Hm_3, s_3) \models \CTLforget(\EXIST \NEXT \phi, V)$.

 (iii) and (iV) can be proved as (i) and (ii) respectively.
 \end{proof}

 \noindent\textbf{Section \ref{forgetting}.4 Complexity Results}
\\

\noindent \textbf{Proposition}\ref{modelChecking}  \textbf{(Model Checking on Forgetting)}  Given an initial \MPK-structure $(\Hm,s_0)$, $V\subseteq{\cal A}$ and $\varphi \in \CTL_{\ALL\FUTURE}$,  deciding $(\Hm,s_0) \models^? \CTLforget(\varphi, V)$ is \textsc{NP}-complete.
\begin{proof}
Membership:
Assume that $(\Hm,s_0)\models\CTLforget(\varphi, V )$, then
there must be an initial \MPK-structure $(\Hm', s_0')$ such that (a) $(\Hm', s_0')\models\varphi$ and
(b) $(\Hm, s_0) \leftrightarrow_V (\Hm', s_0')$. Recall that the condition (a) can be checked in polynomial time in the size of $\Hm'$ and $\varphi$~\cite{DBLP:books/daglib/0007403}. We can also show that it takes polynomial time to check the condition (b) in a similar manner to the proof of Corollary 7.45 in~\cite{Baier:PMC:2008}. Thus, this problem is in $\textsc{NP}$ since guessing such an initial \MPK-structure $(\Hm', s_0')$ which is polynomial in the size of $(\Hm,s_0)$ can be done in polynomial time.
The hardness follows from the fact that the model checking for propositional variable
forgetting is $\textsc{NP}$-hard~\cite{Zhang2008Properties} (considering that propositional variable
forgetting is a special case of forgetting by Theorem~\ref{thm:PL:CTL}).


\end{proof}

\noindent\textbf{Theorem}~\ref{thm:comF} \textbf{(Entailment)} Let $\varphi$ and $\psi$ be two $\CTL_{\ALL \FUTURE}$ formulas and $V$ be a set of atoms. Then,
\begin{enumerate}[(i)]
  \item deciding  $\CTLforget(\varphi, V ) \models^? \psi$ is co-$\textsc{NP}$-complete,
  \item deciding  $\psi \models^? \CTLforget(\varphi, V)$ is $\Pi_2^{\textsc{P}}$-complete,
  \item deciding $\CTLforget(\varphi, V) \models^? \CTLforget(\psi, V)$ is $\Pi_2^{\textsc{P}}$-complete.
\end{enumerate}
\begin{proof}
(i) It is known that deciding whether $\varphi$ is satisfiable is $\textsc{NP}$-Complete~\cite{meier2009complexity}. The hardness follows by setting $\CTLforget(\varphi, \Var(\varphi))\equiv \top$, i.e., deciding whether $\psi$ is valid.
Concerning membership, by Theorem~\ref{thm:close}, we have $\CTLforget(\varphi, V ) \models \psi$ iff $\varphi \models \psi$ and $\IR(\psi, V)$.
Clearly, in $\CTL_{\ALL \FUTURE}$, deciding $\varphi\models \psi$ is in co-$\textsc{NP}$~\cite{meier2009complexity}. We show that deciding whether $\IR(\psi, V )$ is also
in co-NP. W.l.o.g., we assume that $\psi$ is satisfiable.
 Then $\psi$ has a model in the polynomial size of $\psi$.
 We consider the complement of the problem: deciding whether $\psi$ is \emph{not} irrelevant to $V$ (or \emph{relevant}) i.e., $\neg \IR(\psi, V)$. It is easy to see that $\neg \IR(\psi, V)$ iff there exists a model $(\Hm, s_0)$ of $\psi$ and an
initial \MPK-structure $(\Hm',s_0')$ which has a polynomial size in the size of $\psi$ such that
$(\Hm, s_0) \leftrightarrow_V (\Hm',s_0')$ and $(\Hm',s_0')\not \models \psi$. So deciding $\neg \IR(\psi, V)$ can be achieved in two steps: (1) guess two initial \MPK-structures $(\Hm,s_0)$ and $(\Hm',s_0')$ which is of polynomial size   in the size of $\psi$ such that $(\Hm,s_0) \models \psi$ and $(\Hm',s_0')\not \models \psi$, and (2) check
$(\Hm, s_0) \leftrightarrow_V (\Hm',s_0')$. Obviously, both (1) and (2) can be done in polynomial time.

(ii) Membership: We consider the complement of the
problem. We may guess an initial \MPK-structure $(\Hm, s_0)$ which has  polynomial size in the size of $\psi$ satisfying $\psi$ and check whether $(\Hm,s_0)$ $\not \models \CTLforget($ $\varphi$, $V)$. By Proposition~\ref{modelChecking}, we know that it is in $\Sigma_2^{\textsc{P}}$. So the original problem is in $\Pi_2^{\textsc{P}}$. Hardness: Let $\psi \equiv \top$. Then the problem is reduced to decide the validity of  $\CTLforget(\varphi, V )$. Since propositional forgetting is a special case (of forgetting in \CTL) by Theorem~\ref{thm:PL:CTL}, the hardness is directly followed from the proof of Proposition 24 in~\cite{DBLP:journals/jair/LangLM03}.

(iii) Membership: Assume that $\CTLforget(\varphi, V) \not \models \CTLforget(\psi, V)$. Then, there exists an initial \MPK-structure $(\Hm, s)$ such that $(\Hm, s)\models \CTLforget(\varphi, V)$ but $(\Hm, s) \not \models \CTLforget(\psi, V)$, i.e., there is a $(\Hm_1, s_1)$ with $(\Hm_1, s_1) \lrto_V (\Hm, s)$ such that $(\Hm_1, s_1) \models \varphi$ but  for every $(\Hm_2, s_2)$ with $(\Hm, s) \lrto_V (\Hm_2, s_2)$ where $(\Hm_2, s_2) \not \models \psi$. Observe  that such $(\Hm, s)$ and $(\Hm_1, s_1)$ (with the corresponding testing conditions) can be computed in polynomial time in the size of $\varphi, \psi$ and $V$ (since the tasks (a) and (b) in the proof of Proposition~\ref{modelChecking} can be performed in polynomial time).
It is obvious that guessing such $(\Hm, s)$, $(\Hm_1, s_1)$ in the polynomial size of $\varphi$ with $(\Hm_1, s_1) \lrto_V (\Hm, s)$ and checking $(\Hm_1, s_1)\models \varphi$ are feasible while checking $(\Hm_2, s_2) \not \models \psi$ for every $(\Hm, s) \lrto_V (\Hm_2, s_2)$ can be done in polynomial time in the size of $\psi$, and $\Hm_2$.

This shows that the problem is in $\Pi_2^{\textsc{P}}$.

Hardness: It follows from (ii) due to the fact that $\CTLforget(\varphi, V) \models \CTLforget(\psi, V)$ iff $\varphi \models \CTLforget(\psi, V)$ by $\IR(\CTLforget(\psi, V), V)$.

\end{proof}

 \noindent\textbf{Section~\ref{ns_conditions} Necessary and Sufficient Conditions}\\

 \noindent\textbf{Proposition}~\ref{dual} \textbf{(dual)} Let $V,q,\varphi$ and $\psi$ are like in Definition~\ref{def:NC:SC}.
  The $\psi$ is a SNC (WSC) of $q$ on $V$ under $\varphi$ iff $\neg \psi$ is a WSC (SNC)
     of $\neg q$ on $V$ under $\varphi$.
 \begin{proof}
      (i) Suppose $\psi$ is the SNC of $q$. Then $\varphi \models q \rto \psi$. Thus $\varphi \models \neg \psi \rto \neg q$. So $\neg \psi$ is a
 SC of $\neg q$. Suppose $\psi'$ is any other SC of $\neg q$: $\varphi \models \psi' \rto \neg q$. Then $\varphi \models q \rto \neg \psi'$, this means $\neg \psi'$ is a NC of $q$ on $V$ under $\varphi$.
 Thus $\varphi \models \psi \rto \neg \psi'$ by the assumption. So $\varphi \models \psi' \rto \neg \psi$. This proves that $\neg \psi$ is the WSC of $\neg q$.
 The proof of the other part of the proposition is similar.

 (ii) The WSC case can be proved similarly with SNC case.
     \end{proof}

 \noindent\textbf{Proposition}~\ref{formulaNS_to_p}   Let $\Gamma$ and $\alpha$ be two formulas, $V \subseteq \Var(\alpha) \cup \Var(\Gamma)$  and $q$ be a new proposition not in $\Gamma$ and $\alpha$.
  Then, a formula $\varphi$ of $V$ is the SNC (WSC) of $\alpha$ on $V$ under  $\Gamma$ iff it is the SNC (WSC) of $q$ on $V$ under $\Gamma' = \Gamma \cup \{q \lrto \alpha\}$.

 \begin{proof}
     We prove this for SNC. The case for WSC is similar.
     Let $\emph{SNC}(\varphi,\alpha,V,\Gamma)$ denote that $\varphi$ is the SNC of $\alpha$ on $V$ under $\Gamma$, and  $\emph{NC}(\varphi,\alpha,V,\Gamma)$ denote that $\varphi$ is the NC of $\alpha$ on $V$ under $\Gamma$.

     ($\Rto$) We will show that if $\emph{SNC}(\varphi,\alpha,V,\Gamma)$ holds, then $\emph{SNC}(\varphi,q,V,\Gamma')$ will be true. According to $\emph{SNC}(\varphi,\alpha,V,\Gamma)$ and $\alpha\equiv q$, we have $\Gamma' \models q\rto \varphi$, which means $\varphi$ is a NC of $q$ on $V$ under $\Gamma'$. Suppose $\varphi'$ is any NC of $q$ on $V$ under $\Gamma'$, then $\CTLforget(\Gamma',q)\models \alpha \rto \varphi'$ due to $\alpha\equiv q$, $\emph{IR}(\alpha \rto \varphi', \{q\})$ and $(\PP)$, i.e., $\Gamma \models \alpha \rto \varphi'$ by Lemma \ref{lem:KF:eq}, this means $\emph{NC}(\varphi',\alpha,V,\Gamma)$. Therefore, $\Gamma \models \varphi \rto \varphi'$ by the definition of SNC and $\Gamma' \models \varphi \rto \varphi'$. Hence, $\emph{SNC}(\varphi,q,V,\Gamma')$ holds.

     ($\Lto$) We will show that if $\emph{SNC}(\varphi,q,V,\Gamma')$ holds, then $\emph{SNC}(\varphi,\alpha,V,\Gamma)$ will be true. According to $\emph{SNC}(\varphi,q,V,\Gamma')$, it's not difficult to know that $\CTLforget(\Gamma', \{q\})\models \alpha \rto \varphi$ due to $\alpha\equiv q$, $\emph{IR}(\alpha \rto \varphi, \{q\})$ and $(\PP)$, i.e., $\Gamma \models \alpha \rto \varphi$ by Lemma \ref{lem:KF:eq}, this means $\emph{NC}(\varphi,\alpha,V,\Gamma)$. Suppose $\varphi'$ is any NC of $\alpha$ on $V$ under $\Gamma$. Then $\Gamma' \models q \rto \varphi'$ since $\alpha\equiv q$ and $\Gamma'=\Gamma \cup \{q\equiv \alpha\}$, which means $\emph{NC}(\varphi',q,V,\Gamma')$. According to $\emph{SNC}(\varphi,q,V,\Gamma')$, $\emph{IR}(\varphi \rto \varphi', \{q\})$ and $(\PP)$, we have
     $\CTLforget(\Gamma', \{q\})\models \varphi \rto \varphi'$, and $\Gamma \models \varphi \rto \varphi'$ by Lemma \ref{lem:KF:eq}. Hence, $\emph{SNC}(\varphi,\alpha,V, \Gamma)$ holds.
     \end{proof}

 \noindent\textbf{Theorem}~\ref{thm:SNC:WSC:forget} Let $\varphi$ be a formula, $V\subseteq\Var(\varphi)$ and $q\in\Var(\varphi)- V$.
  \begin{enumerate}[(i)]
   \item $\CTLforget (\varphi \land q$, $(\Var(\varphi) \cup \{q\})- V)$
   is a SNC of $q$ on $V$ under $\varphi$.
   \item  $\neg\CTLforget (\varphi \land \neg q$, $(\Var(\varphi) \cup \{q\})- V)$
   is a WSC of $q$ on $V$ under $\varphi$.
  \end{enumerate}
 \begin{proof}
  We will prove the SNC part, while it is not difficult to prove the WSC part according to Proposition \ref{dual}.
  Let ${\cal F}=\CTLforget(\varphi \wedge q, (\Var(\varphi) \cup \{q\})- V)$.

   The ``NC" part: It's easy to see that $\varphi \wedge q \models {\cal F}$ by {\bfseries (W)}. Hence, $\varphi\models q \rto {\cal F}$, this means
   ${\cal F}$ is a NC of $q$ on $V$ under $\varphi$.



     The ``SNC" part: We will show that for all NC $\psi'$ of $q$ on $V$ under $\varphi$ (i.e $\varphi\models q \rto \psi'$) there is $\varphi \models {\cal F} \rto \psi'$.
    We know that if $\varphi \wedge q \models \psi'$ then ${\cal F} \models \psi'$ by {\bfseries (PP)} due to $\emph{IR}(\psi', (\Var(\varphi) \cup \{q\})- V)$. Therefore, we have $\varphi \wedge {\cal F} \models \psi'$ since $\psi'$ is a NC of $q$ on $V$ under $\varphi$ and then $\varphi \models {\cal F} \rto \psi'$, i.e.  ${\cal F}$ is the SNC of $q$ on $V$ under $\varphi$.
  \end{proof}

 \noindent \textbf{Theorem}~\ref{thm:inK:SNC} Let ${\cal K}= (\Hm, s)$ be an initial \MPK-structure with $\Hm=(S,R,L,s_0)$ on the set $\Ha$ of atoms, $V \subseteq \Ha$ and $q\in V' = \Ha - V$. Then:
  \begin{enumerate}[(i)]
   \item the SNC of $q$ on $V$ under ${\cal K}$ is $\CTLforget({\cal F}_{\Ha}({\cal K}) \wedge q, V')$.
   \item the WSC of $q$ on $V$ under ${\cal K}$ is $\neg \CTLforget({\cal F}_{\Ha}({\cal K}) \wedge \neg q, V')$.
  \end{enumerate}
 \begin{proof}
 (i)
 As we know that any initial \MPK-structure ${\cal K}$ can be described as a characterizing formula ${\cal F}_{\Ha}({\cal K})$, then the SNC of $q$ on $V$ under ${\cal F}_{\Ha}({\cal K})$ is $\CTLforget({\cal F}_{\Ha}({\cal K}) \wedge q, \Ha - V)$.

 (ii) This is proved by the dual property.
 \end{proof}


 \noindent\textbf{Section \ref{section_algorithm} An Algorithm Computing CTL Forgetting}\\

 \noindent\textbf{Proposition}\ref{pro:time:alg1} Let $\varphi$ be a \CTL\ formula and $V\subseteq \Ha$ with $|{\cal S}|=m$, $|\Ha|=n$ and $|V|=x$. The space complexity is $O((n-x)m^{2(m+2)}2^{nm} * \log m)$ and the time complexity of Algorithm~\ref{alg:compute:forgetting:by:VB} is at least the same as the space.
\begin{proof}
Supposing each state or atom occupy $\log m$ (supposing $n\leq m$), then a state pair $(s, s')$ occupy $2* \log m$ bits.
For any $B\subseteq {\cal S}$ with $B \not = \emptyset$ and $s_0\in B$, we can construct an initial \MPK-structure $(\Hm, s_0)$ with $\Hm=(B, R, L, s_0)$, in which there is at most $\frac{|B|^2}{2}$ state pairs in $R$ and $|B|*n$ pairs $(s, A)$ ($A \subseteq \Ha$) in $L$. Hence, the $(\Hm, s_0)$ occupy at most $(|B|+|B|^2 + |B|*n)*\log m$ bits.
Besides, for the set $B$ of states we have $|B|$ choices for the initial state, $|B|^{|B|}$ choices for the $R$ and $(2^n)^{|B|}$ choices for the $L$. In the worst case, i.e., when $|B|=m$, we have $m*(m^m * 2^{nm} * m)$ number of initial \MPK-structures.
Therefore, there is at most $m^{m+2}*2^{nm}$ number of initial \MPK-structures, hence it will at most cost $(m^{m+2}*2^{nm}*(m+m^2+nm))*\log m$ bits.

Let $k=n-x$, for any initial \MPK-structure ${\cal K}=(\Hm, s_0)$ with $i\geq 1$ nodes and $\Hm=(B, R, L, s_0)$, in the worst case, i.e., when $ch(\Hm,V)=i$, we will spend $N(i)=P_i(s_0) + i * (P_i(s) + i * P_i(s'))$ space to store the characterizing formula of ${\cal K}$ on $\overline{V}$. Where $s', s\in B$ and $P_i(y)$ is the space spend to store ${\cal F}_{\overline{V}}(\Tr_i(y))$ with $y\in B$.
(We suppose the formulas in $\EXIST \NEXT$ and $\ALL \NEXT$ parts share the same memory.)
In the following, we compute inductively the space needed to store the ${\cal F}_{\overline{V}}(\Tr_n(y))$ with $0\leq n \leq i$
\begin{align*}
    & (1)\ n = 0, && P_0(y) = k\\ 
    & (2)\ n =1, && P_1(y) = k + i *k = k + i* P_0(y)\\
    & (3)\ n =2, && P_2(y) = k + i*(k + i*k) = k + i*P_1(y)\\
    & \dots && \dots\\
    & (i+1)\ n = i, && P_i(y) =k+ i *P_{i-1}(y).
\end{align*}
Therefore, we have
\begin{align*}
 P_i(y) &= k + i*k+ i^2*k\dots + i^i * k = \frac{i^i -1}{i-1} k, \text{ and }\\
 N(i)&= P_i(s_0) + i * (P_i(s) + i * P_i(s'))\\
 = & (i^2 +i +1) P_i(y)\\
 =& (i^2+i +1) \frac{i^i -1}{i-1} k.
\end{align*}
In the worst case, i.e., there is $m^{m+2}*2^{nm}$ initial \MPK-structures with $m$ nodes, we will spent $(m^{m+2}*2^{nm}*N(m))*\log m$ bits to store the result of forgetting.

Therefore, the space complexity is $O((n-x)m^{2(m+2)}2^{nm} * \log m)$ and the time complexity is at least the same as the space.
\end{proof}


\clearpage

\end{document}